\documentclass{article}

\PassOptionsToPackage{numbers, sort&compress}{natbib}

    \usepackage[final]{neurips_2021}

\usepackage[utf8]{inputenc} 
\usepackage[T1]{fontenc}    
\usepackage{hyperref}       
\usepackage{url}            
\usepackage{booktabs}       
\usepackage{amsfonts}       
\usepackage{nicefrac}       
\usepackage{microtype}      
\usepackage{xcolor}         

\usepackage{algorithm}
\usepackage[noend]{algorithmic}

\usepackage{authblk} 

\usepackage{multirow}
\usepackage{amssymb}
\usepackage{amsthm}
\usepackage{amsmath}
\usepackage{mathtools}

\newtheorem{theorem}{Theorem}
\newtheorem{corollary}[theorem]{Corollary}
\newtheorem{lemma}[theorem]{Lemma}
\newtheorem{proposition}[theorem]{Proposition}

\newtheorem{definition}[theorem]{Definition}
\newtheorem{assumption}[theorem]{Assumption}  
\theoremstyle{remark}
\theoremstyle{claim}
\newtheorem{claim}{Claim}[section]
\newtheorem*{remark}{Remark}

\usepackage{wrapfig,lipsum,booktabs}
\usepackage{caption} 
\usepackage{enumerate}
\usepackage{changepage}
\usepackage{subfig}
\usepackage{listings}
\usepackage[hashEnumerators,smartEllipses]{markdown}
\markdownSetup{fencedCode = true}

\title{Fault-Tolerant Federated Reinforcement Learning with Theoretical Guarantee}

\author[1, 3]{\textbf{Flint Xiaofeng Fan}}
\author[2]{\textbf{Yining Ma}}
\author[1]{\textbf{Zhongxiang Dai}}
\author[4]{\textbf{Wei Jing}}
\author[3]{\\\textbf{Cheston Tan}}
\author[1]{\textbf{Bryan Kian Hsiang Low}}

\affil[1]{Dept. of Computer Science, National University of Singapore, Republic of Singapore}
\affil[2]{Dept. of ISEM, National University of Singapore, Republic of Singapore}
\affil[3]{Institute for Infocomm Research, A*STAR, Republic of Singapore}
\affil[4]{Alibaba DAMO Academy, Hangzhou, China}

\affil[1]{\texttt{\{xiaofeng,daizhongxiang,lowkh\}@comp.nus.edu.sg,$^2$yiningma@u.nus.edu}}
\affil[3]{\texttt{\{stufanxf,cheston-tan\}@i2r.a-star.edu.sg,$^4$jw334405@alibaba-inc.com}}

\begin{document}

\maketitle

\begin{abstract}
    The growing literature of \textit{Federated Learning} (FL) has recently inspired \textit{Federated Reinforcement Learning} (FRL) to encourage multiple agents to federatively build a \emph{better} decision-making policy without sharing raw trajectories. Despite its promising applications, existing works on FRL fail to I) provide theoretical analysis on its convergence, and II) account for random system failures and adversarial attacks. Towards this end, we propose the first FRL framework the convergence of which is guaranteed and tolerant to less than half of the participating agents being random system failures or adversarial attackers. We prove that the sample efficiency of the proposed framework is guaranteed to improve with the number of agents and is able to account for such potential failures or attacks. All theoretical results are empirically verified on various RL benchmark tasks.
    Our code is available at \href{https://github.com/flint-xf-fan/Byzantine-Federeated-RL}{https://github.com/flint-xf-fan/Byzantine-Federeated-RL.}
\end{abstract}


\section{Introduction}\label{sec:introduction}

\textit{Reinforcement learning} (RL) has recently been applied to many real-world decision-making problems such as gaming, robotics, healthcare, etc.~\citep{mnih2013playing,sergey2015learning,RL-MIMIC-survey}.
However, despite its impressive performances in simulation, RL often suffers from poor sample efficiency, which hinders its success in
real-world applications \citep{dulac2019challengesRealWorld-RL,levine2020offlineRL}. 
For example, when RL is applied to provide clinical decision support \citep{Nature-RL-MIMIC,RL-MIMIC-survey,DDPG-RL-MIMIC},
its performance is limited by the number (i.e., sample size) of admission records possessed by a hospital, which cannot be synthetically generated \citep{RL-MIMIC-survey}.
As this challenge is usually faced by many agents (e.g., different hospitals), a natural solution is to encourage multiple RL agents to share their trajectories, to collectively build a better decision-making policy that one single agent can not obtain by itself. 
However, in many applications, raw RL trajectories contain sensitive information (e.g., the medical records contain sensitive information about patients) and thus sharing them is prohibited.
To this end, the recent success of \textit{Federated Learning} (FL) \citep{konevcny2016federated,kairouz2019FedLearn,li2019survey,li2020practical} has inspired the setting of \textit{Federated Reinforcement Learning} (FRL) \citep{zhuo2019federatedRLYangQ}, which aims to \textit{federatively} build a \emph{better} policy from multiple RL agents without requiring them to share their raw trajectories.
FRL is practically appealing for addressing the sample inefficiency of RL in real systems, such as autonomous driving~\citep{liang2019FedRL-for-AV}, fast personalization~\citep{nadiger2019federatedRL}, optimal control of IoT devices~\citep{lim2020federated-Sensors}, robots navigation~\citep{liu2019FedRL-for-robots}, and resource management in networking~\citep{yu2020FedRLfor5G}. Despite its promising applications, FRL is faced by a number of major challenges, which existing works are unable to tackle.

Firstly, existing FRL frameworks are not equipped with theoretical convergence guarantee, and thus lack an assurance for the sample efficiency of practical FRL applications, which is a critical drawback due to the high sampling cost of RL trajectories in real systems \citep{dulac2019challengesRealWorld-RL}. Unlike FL where training data can be collected offline, FRL requires every agent to sample trajectories by interacting with the environment during learning.
However, interacting with real systems can be slow, expensive, or fragile. This makes it critical for FRL to be sample-efficient and hence highlights the requirement for convergence guarantee of FRL, without which no assurance on its sample efficiency is provided for practical applications.
To fill this gap, we establish on recent endeavors in stochastic variance-reduced optimization techniques to develop a 
variance-reduced 
federated policy gradient framework, the convergence of which is guaranteed. We prove that the proposed framework enjoys a sample complexity of $O(1/\epsilon^{5/3})$ to converge to an $\epsilon$-stationary point
in the single-agent setting, which matches recent results 
of variance-reduced policy gradient \citep{papini2018stochastic,xu2020improvedUAI}. 
More importantly, the aforementioned sample complexity is guaranteed to \emph{improve} at a rate of $O(1/K^{2/3})$ upon the federation of $K$ agents.
This guarantees that an agent achieves a better sample efficiency by joining the federation and benefits from more participating agents, which are highly desirable in FRL.

Another challenge inherited from FL is that FRL is vulnerable to random failures or adversarial attacks, which poses threats to many real-world RL systems. 
For example, robots may behave arbitrarily due to random hardware issues; 
clinical data may provide inaccurate records and hence create misleading trajectories~\citep{RL-MIMIC-survey};
autonomous vehicles, on which RL is commonly deployed, are subject to adversarial attacks~\citep{cao2019adversarialAttack-AV}.
As we will show in experiments, including such random failures or adversary agents in FRL can significantly deteriorate its convergence or even result in unlearnability.
{Of note, random failures and adversarial attacks in FL systems are being encompassed by the Byzantine failure model~\citep{lamport1982byzantine}, which is considered as the most stringent fault formalism in distributed computing~\citep{distributed-algorithms-book,castro1999practicalByzantine} -- a small fraction
of agents may behave arbitrarily and possibly adversarially, with the goal of breaking or at least slowing down the convergence of the system.} As algorithms proven to be correct in this setting are guaranteed to converge under arbitrary system behavior (e.g., exercising failures or being attacked)~\citep{kairouz2019FedLearn,li2020federatedSurvey2}, {we study the fault tolerance of our proposed FRL framework using the Byzantine failure model. We design a gradient-based Byzantine filter on top of the variance-reduced federated policy gradient framework.}
We show that, when a certain percentage (denoted by $\alpha < 0.5$) of agents are Byzantine agents, the sample complexity of the FRL system is worsened by \textit{only an additive term} of $O(\alpha^{4/3}/\epsilon^{5/3})$ (Section~\ref{section:theoretical-results}).
Therefore, when $\alpha \rightarrow 0$, (i.e., an ideal system with zero chance of failure), the filter induces no impact on the convergence.

\begin{figure}[hb]
    \centering
    \includegraphics[width=\textwidth]{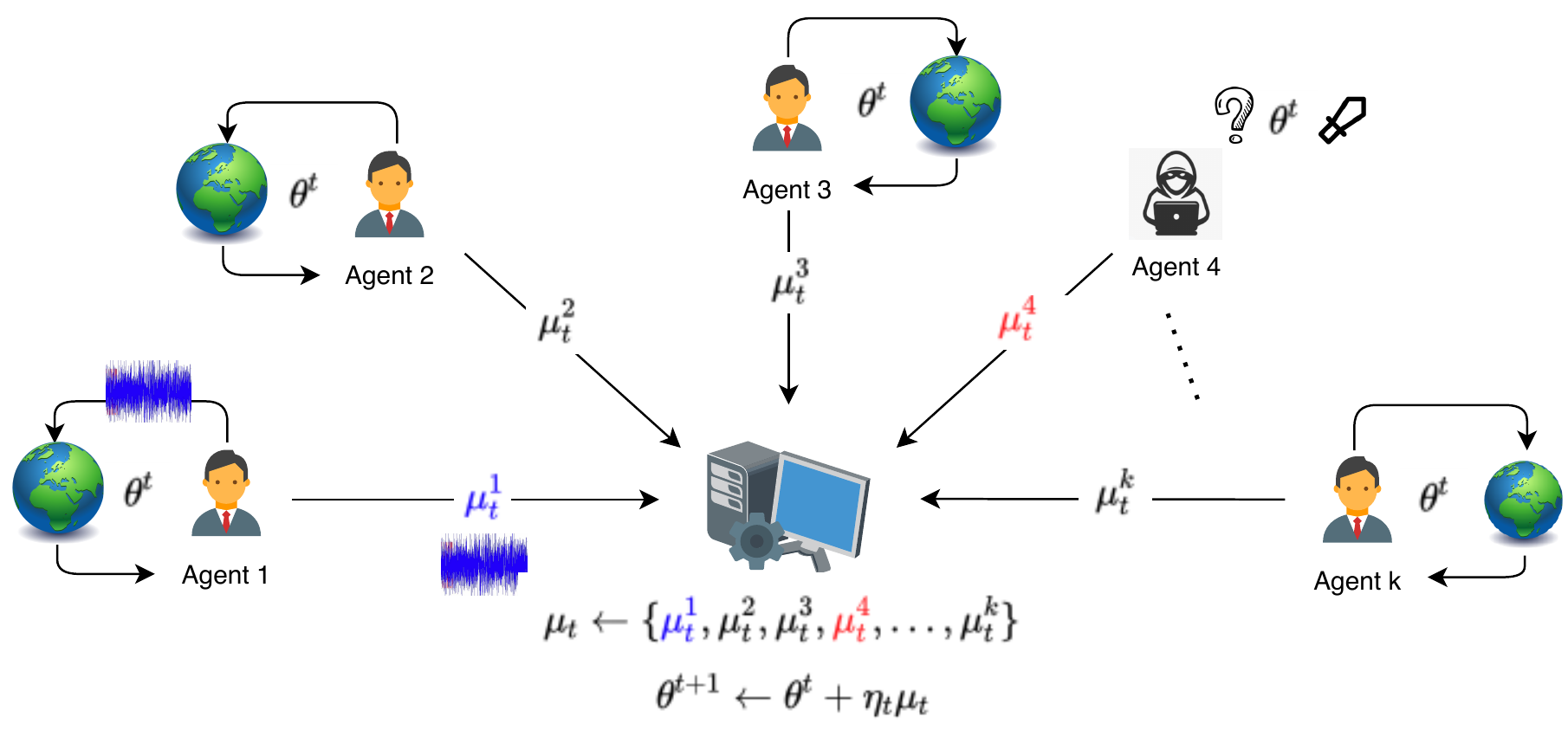}
    \caption{Graphical illustration of Federated Reinforcement Learning with the potential presence of Byzantine agents. Each agent \emph{independently} operates in a \emph{separate} copy of the underlying MDP and communicates with the central server. Agents are subject to failures or adversarial attacks from time to time. The server does not know when an agent turns into a Byzantine agent.}
    \label{fig:byzantine_FRL}
\end{figure}

\textbf{Contributions.} In this paper, we study the \emph{federated} reinforcement learning problem with theoretical guarantee in the potential presence of faulty agents, as illustrated in Fig.~\ref{fig:byzantine_FRL}. 
We introduce \emph{Federated Policy Gradient with Byzantine Resilience} (FedPG-BR), the first FRL framework that is theoretically principled and practically effective for the FRL setting, accounting for random systematic failures and adversarial attacks. In particular,
FedPG-BR (a) enjoys a guaranteed sample complexity which improves with more participating agents, and (b) is tolerant to the Byzantine fault in both theory and practice. We discuss the details of problem setting and the technical challenges (Section \ref{section:settings}) and provide theoretical analysis of FedPG-BR (Section \ref{section:theoretical-results}). We also demonstrate its empirical efficacy on various RL benchmark tasks (Section \ref{section:experiments}).

\section{Background}\label{sec:background}
\textbf{Stochastic Variance-Reduced Gradient} aims to solve $\min_{\boldsymbol{\theta} \in \mathbb{R}^d}[J(\boldsymbol{\theta}) \triangleq \frac{1}{B} \sum_{i=1}^{B} J_{i}(\boldsymbol{\theta})]$.
Under the common assumption of all function components $J_i$ being smooth and convex in $\boldsymbol{\theta}$, \textit{gradient descent} (GD) achieves linear convergence in the number of iterations of parameter updates \citep{cauchy1847GD, nesterov2013introductory}.
However, every iteration of GD requires $B$ gradient computations, which can be expensive for large $B$. 
To overcome this problem, \textit{stochastic GD} (SGD)~\citep{robbins1951stochastic,bottou2003sgd} samples a single data point per iteration, which incurs lower per-iteration cost yet results in a sub-linear convergence rate \citep{nemirovsky1983sgdConvergence}. 
For a better trade-off between convergence rate and per-iteration computational cost, the \textit{stochastic variance-reduced gradient} (SVRG) method has been proposed, which reuses past gradient computations to reduce the variance of the current gradient estimate~\citep{johnson2013svrg1,xiao2014svrg4,allen2016svrg2,reddi2016svrg3}.
More recently, \textit{stochastically controlled stochastic gradient} (SCSG) has been proposed for convex~\citep{lei2016scsg1} or smooth non-convex objective function \citep{lei2017scsg2}, to further reduce the computational cost of SVRG especially when required $\epsilon$ is small in finding $\epsilon$-approximate solution.
Refer to Appendix~\ref{Appendix-SVRG-SCSG} for more details on SVRG and SCSG.

\textbf{Reinforcement Learning} (RL) can be modelled as a discrete-time Markov Decision Process (MDP)~\citep{sutton2018reinforcement}: $M \triangleq \{\mathcal{S}, \mathcal{A}, \mathcal{P}, \mathcal{R}, \gamma, \rho\}$. $\mathcal{S}$ represents the state space, $\mathcal{A}$ is the action space, $\mathcal{P}(s^\prime | s, a)$ defines the transition probability from state $s$ to $s^\prime$
after taking action $a$, $\mathcal{R}(s,a): \mathcal{S} \times \mathcal{A} \mapsto[0, R] $ is the reward function for state-action pair $(s,a)$ and some constant $R > 0$, $\gamma \in (0,1)$ is the discount factor, and $\rho$ is the initial state distribution. 
An agent's behavior is controlled by a policy $\pi$, where $\pi(a|s)$ defines the probability that the agent chooses action $a$ at state $s$. 
We consider episodic MDPs with trajectory horizon $H$. A trajectory $\tau \triangleq \{s_0, a_0, s_1, a_1, ..., s_{H-1}, a_{H-1}\}$ is a sequence of state-action pairs traversed by an agent following any stationary policy, where $s_0 \sim \rho$. $\mathcal{R}(\tau) \triangleq\sum_{t=0}^{H-1} \gamma^{t} \mathcal{R}\left(s_{t}, a_{t}\right)$ gives the cumulative discounted reward for a trajectory $\tau$.

\textbf{Policy Gradient} (PG) methods have achieved impressive successes in model-free RL \citep[][etc.]{schulman2015TRPO,schulman2017PPO}. 
Compared with deterministic value-function based methods such as Q-learning, PG methods are generally more effective in high-dimensional problems and enjoy the flexibility of stochasticity. 
In PG, we use $\pi_{\boldsymbol{\theta}}$ to denote the policy parameterized by $\boldsymbol{\theta} \in \mathbb{R}^d$ (e.g., a neural network), and $p(\tau | \pi_{\boldsymbol{\theta}})$ to represent the trajectory distribution induced by policy $\pi_{\boldsymbol{\theta}}$. 
For brevity, we use $\theta$ to denote the corresponding policy $\pi_{\boldsymbol{\theta}}$.
The performance of a policy $\boldsymbol{\theta}$ can be measured by $J(\boldsymbol{\theta}) \triangleq \mathbb{E}_{\tau \sim p(\cdot | \boldsymbol{\theta})}[\mathcal{R}(\tau)|M]$.
Taking the gradient of $J(\boldsymbol{\theta})$ with respect to $\boldsymbol{\theta}$ gives
\begin{align*}
\nabla_{\boldsymbol{\theta}} J(\boldsymbol{\theta}) =\int_{\tau} \mathcal{R}(\tau) \nabla_{\boldsymbol{\theta}} p(\tau \mid \boldsymbol{\theta}) \mathrm{d} \tau 
=\mathbb{E}_{\tau \sim p(\cdot \mid \boldsymbol{\theta})}\left[\nabla_{\boldsymbol{\theta}} \log p(\tau \mid \boldsymbol{\theta}) \mathcal{R}(\tau) \mid M\right] \stepcounter{equation}\tag{\theequation}\label{def-full-gradient}
\end{align*}
Then, the policy $\boldsymbol{\theta}$ can be optimized by gradient ascent. 
Since computing~\eqref{def-full-gradient} is usually prohibitive, stochastic gradient ascent is typically used. In each iteration, we sample a batch of trajectories $\{\tau_i\}_{i=1}^B$ using the current policy $\boldsymbol{\theta}$, and update the policy by $\boldsymbol{\theta} \leftarrow \boldsymbol{\theta} + \eta \widehat{\nabla}_{B} J\left(\boldsymbol{\theta}\right)$, where $\eta$ is the step size and $\widehat{\nabla}_{B} J(\boldsymbol{\theta})$ is an estimate of \eqref{def-full-gradient} using the sampled trajectories $\{\tau_i\}_{i=1}^B$: $\widehat{\nabla}_{B} J(\boldsymbol{\theta})=\frac{1}{B} \sum_{i=1}^{B} \nabla_{\boldsymbol{\theta}} \log p\left(\tau_{i} \mid \boldsymbol{\theta}\right) \mathcal{R}\left(\tau_{i}\right)$.
The most common policy gradient estimators, such as REINFORCE \citep{williams1992REINFORCE} and GPOMDP \citep{baxter2001GPOMDP}, can be expressed as 
\begin{equation}
    \label{def-gradient-estimate}
    \widehat{\nabla}_{B} J(\boldsymbol{\theta})=\frac{1}{B} \sum_{i=1}^{B} g(\tau_i | \boldsymbol{\theta})
\end{equation}
where $\tau_{i}=\{s_{0}^{i}, a_{0}^{i}, s_{1}^{i}, a_{1}^{i}, \ldots, s_{H-1}^{i}, a_{H-1}^{i}\}$ and $g(\tau_i | \boldsymbol{\theta})$ is an \textit{unbiased} estimate of $\nabla_{\boldsymbol{\theta}} \log p(\tau_{i} \mid  \boldsymbol{\theta}) \mathcal{R}(\tau_{i})$. We provide formal definition of $g(\tau_i | \boldsymbol{\theta})$ in Appendix \ref{appendix-gradient-estimator}.

\textbf{SVRPG.} 
A key issue for PG is the high variance of the estimator based on stochastic gradients~\eqref{def-gradient-estimate} which results in slow convergence.
Similar to SGD for finite-sum optimization, PG requires $O(1/\epsilon^2)$ trajectories to find an $\epsilon$-stationary point such that $\mathbb{E}[\|\nabla J(\boldsymbol{\theta})\|^2] \leq \epsilon$~\citep{xu2020improvedUAI}. 
That is, PG typically requires a large number of trajectories to find a well-performing policy.
To reduce the variance of the gradient estimator in PG~\eqref{def-gradient-estimate}, SVRG has been applied to policy evaluation~\citep{du2017SVRGPolicyEval,ijcai2020-0374} and policy optimization~\citep{xu2017SVRGTRPO}.
The work of~\citet{papini2018stochastic} has adapted the theoretical analysis of SVRG to PG to introduce the \textit{stochastic variance-reduced PG} (SVRPG) algorithm.
More recently,~\citet{xu2020improvedUAI} has refined the analysis of SVRPG~\citep{papini2018stochastic} and shown that SVRPG enjoys a sample complexity of $O(1/\epsilon^{5/3})$.
These works have demonstrated both theoretically and empirically that SVRG is a promising approach to reduce the variance and thus improve the sample efficiency of PG methods.

\textbf{Fault tolerance} refers to the property that enables a computing system to continue operating properly without interruption when one or more of its workers fail. Among the many fault formalisms, the Byzantine failure model has a rich history in distributed computing~\citep{distributed-algorithms-book,castro1999practicalByzantine} and is considered as the most stringent fault formalism in fault-tolerant FL system design~\citep{kairouz2019FedLearn,li2020federatedSurvey2}. 
Originated from the \textit{Byzantine generals problem}~\citep{lamport1982byzantine}, the Byzantine failure model allows an $\alpha$-fraction (typical $\alpha < 0.5$) of workers to behave arbitrarily and possibly adversarially, with the goal of breaking or at least slowing down the convergence of the algorithm.
As algorithms proven to be resilient to the Byzantine failures are guaranteed to converge under arbitrary system behavior (hence fault-tolerant)~\citep{distributed-algorithms-book,castro1999practicalByzantine}, it has motivated a significant interest in providing distributed \emph{supervised learning} with Byzantine resilience guarantees \cite[e.g.,][]{blanchard2017machine-Krum,yin2018byzantine,alistarh2018byzantine,baruch2019Byzantine-VA-attack,khanduri2019byzantine,allen2020byzantine-iclr}.
However, there is yet no existing work studying the correctness of Byzantine resilience in the context of FRL.

\section{Fault-tolerant federated reinforcement learning}
\label{section:settings}

\subsection{Problem statement}\label{subsec:problem-setting}
Our problem setting is similar to that of FL~\citep{konevcny2016federated} where a central server is assumed to be trustworthy and governs the federation of $K$ distributed agents $k \in \{1,...,K\}$. 
In each round $t \in \{1,...,T\}$, the central server broadcasts its parameter $\boldsymbol{\theta}_0^t$ to all agents. 
Each agent then independently samples a batch of trajectories $\{\tau^{(k)}_{t,i}\}_{i=1}^{B_t}$ by interacting with the environment using the obtained policy,  e.g., $\{\tau^{(k)}_{t,i}\}_{i=1}^{B_t} \sim p(\cdot | \boldsymbol{\theta}_0^t)$. 
However, different from FL where each agent computes the parameter updates and sends the updated parameter to the server for aggregation~\citep{konevcny2016federated}, agents in our setup do not compute the updates locally, but instead send the gradient computed w.r.t.\ their local trajectories $\mu^{(k)}_t \triangleq \widehat{\nabla}_{B_t} J(\boldsymbol{\theta}_0^t)$ directly to the server. The server then aggregates the gradients, performs a policy update step, and starts a new round of federation. 

Of note, every agent including the server is operating in a separate copy of the MDP. No exchange of raw trajectories is required, and no communication between any two agents is allowed. To account for potential failures and attacks, we allow an $\alpha$-fraction of agents to be Byzantine agents with $\alpha \in [0, 0.5)$. 
That is, in each round $t$, a good agent always sends its computed $\mu^{(k)}_t$ back to the server, while a Byzantine agent may return any arbitrary vector.\footnote{A Byzantine agent may not be Byzantine in every round.} The server has no information regarding whether Byzantine agents exist and cannot track the communication history with any agent.
In every round, the server can only access the $K$ gradients received from agents, and thereby uses them to detect Byzantine agents so that it only aggregates the gradients from those agents that are believed to be non-Byzantine agents.

\textbf{Notations.} Following the notations of SCSG~\citep{lei2017scsg2}, we use $\boldsymbol{\theta}_0^t$ to denote the server's initial parameter in round $t$ and $\boldsymbol{\theta}_n^t$ to represent the updated parameter at the $n$-th step in round $t$.
$\tau_{t,i}^{(k)}$ represents agent $k$'s $i$-th trajectory sampled using $\boldsymbol{\theta}^t_0$.
$\|\cdot\|$ denotes Euclidean norm and Spectral norms for vectors and matrices, respectively. $O(\cdot)$ hides all constant terms.

\subsection{Technical challenges}\label{subsection:technical-challenges}
There is an emerging interest in Byzantine-resilient distributed \textit{supervised learning} \cite[e.g.,][]{blanchard2017machine-Krum,xie2019zeno-byzantine,yin2018byzantine,alistarh2018byzantine,baruch2019Byzantine-VA-attack,khanduri2019byzantine,allen2020byzantine-iclr}. However, a direct application of those works to FRL is not possible due to that the objective function $J(\boldsymbol{\theta})$ of RL, which is conditioned on $\tau \sim p(\cdot | \boldsymbol{\theta})$, is different from the supervised classification loss seen in the aforementioned works, resulting in the following issues:

\emph{Non-stationarity}: unlike in supervised learning, the distribution of RL trajectories is affected by the value of the policy parameter which changes over time (e.g., $\tau \sim p(\cdot | \boldsymbol{\theta})$). We deal with the non-stationarity using importance sampling \citep{2018importanceSampling} (Section \ref{subsection:algorithm-description}).

\emph{Non-concavity}: the objective function $J(\boldsymbol{\theta})$ is typically non-concave.
To derive the theoretical results accounting for the non-concavity, we need the $L$-smoothness assumption on $J(\boldsymbol{\theta})$, which is a reasonable assumption and commonly made in the literature \cite{pirotta2015Smooth} (Section \ref{section:theoretical-results}). Hence we aim to find an $\epsilon$-approximate solution (i.e., a commonly used objective in non-convex optimization):
\begin{definition}[$\epsilon$-approximate solution]\label{definition:epsilon-approximate-solution}
A point $\boldsymbol{\theta}$ is called $\epsilon$-stationary if $\|\nabla J(\boldsymbol{\theta})\|^2 \leq \epsilon$. Moreover, the algorithm is said to achieve an $\epsilon$-approximate solution in $t$ rounds if $\mathbb{E}[\|\nabla J(\boldsymbol{\theta})\|^2] \leq \epsilon$, where the expectation is with respect to all randomness of the algorithm until round $t$.
\end{definition}
\emph{High variance in gradient estimation:} 
the high variance in estimating \eqref{def-gradient-estimate} renders the FRL system vulnerable to variance-based attacks which conventional Byzantine-resilient optimization works fail to defend \citep{baruch2019Byzantine-VA-attack}. 
To combat this issue, we adapt the SCSG optimization \citep{lei2017scsg2} to federated policy gradient for a refined control over the estimation variance, hence enabling the following assumption which we exploit to design our Byzantine filtering step:
\begin{assumption}[On bounded variance of the gradient estimator]
    \label{assumption-bounded-variance}
    There is a constant $\sigma$ such that $\|g(\tau|\boldsymbol{\theta}) - \nabla J(\boldsymbol{\theta})\| \leq \sigma$ for any $\tau \sim p(\tau | \boldsymbol{\theta})$ for all policy $\pi_{\boldsymbol{\theta}}$.
\end{assumption}
\begin{remark}
{Assumption~\ref{assumption-bounded-variance} is also seen in Byzantine-resilient optimization~\citep{alistarh2018byzantine, khanduri2019byzantine} and may be relaxed to $\mathbb{E}\|g(\tau | \boldsymbol{\theta}) - \nabla J(\boldsymbol{\theta})\| \leq \boldsymbol{\theta}$} which is a standard assumption commonly used in stochastic non-convex optimization \citep[e.g.][]{allen2016variance,lei2017scsg2}. In this work, the value of $\sigma$ is the maximum difference between optimal gradient $\nabla J(\boldsymbol{\theta})$ and the gradient estimate $g(\tau | \boldsymbol{\theta)}$ w.r.t. any trajectories induced by policy $\pi_{\boldsymbol{\theta}}$. For complex real-world problems with continuous, high-dimensional controls, $\sigma$ may be upper-bounded, provided that the MDP is Lipschitz continuous\citet{pirotta2015Smooth}. The deviation can be obtained by referring to Proposition 2 of \citet{pirotta2015Smooth}.\footnote{The value of $\sigma$ can be estimated at the server.}
\end{remark}

\subsection{Algorithm description}\label{subsection:algorithm-description}
The pseudocode for the proposed \emph{Federated Policy Gradient with Byzantine Resilience} (FedPG-BR) is shown in Algorithm~\ref{alg:FedPG-BR}.
FedPG-BR starts with a randomly initialized parameter $\tilde{\boldsymbol{\theta}}_0$ at the server. 
At the beginning of the $t$-th round, the server keeps a snapshot of its parameter from the previous round (i.e., $\boldsymbol{\theta}_0^t \leftarrow \tilde{\boldsymbol{\theta}}_{t-1}$) and broadcasts this parameter to all agents (line 3).
Every (good) agent $k$ samples $B_t$ trajectories $\{\tau_{t,i}^{(k)}\}_{i=1}^{B_t}$ using the policy $\boldsymbol{\theta}_0^t$ (line 5), computes a gradient estimate $\mu_t^{(k)} \triangleq 1/B_t \sum_{i=1}^{B_t}g(\tau_{t,i}^{(k)} | \boldsymbol{\theta}_0^t)$ where $g$ is either the REINFORCE or the GPOMDP estimator (line 6), and sends $\mu_t^{(k)}$ back to the server. 
For a Byzantine agent, it can send an arbitrary vector instead of the correct gradient estimate. 
After all gradients are received, the server performs the \textit{Byzantine filtering} step, and then computes the batch gradient $\mu_t$ by averaging those gradients that the server believes are from non-Byzantine agents (line 7). 
For better clarity, we present the subroutine \textbf{FedPG-Aggregate} for Byzantine filtering and gradient aggregation in Algorithm~\ref{alg:FedAgg-BR}, which we discuss in detail 
separately. 
\begin{algorithm}[b]
   \caption{FedPG-BR}
   \label{alg:FedPG-BR}
\begin{algorithmic}[1]
   \STATE {\bfseries Input:} $\tilde{\boldsymbol{\theta}}_{0} \in \mathbb{R}^{d}$,  batch size $B_t$, mini batch size $b_t$, step size $\eta_{t}$
   \FOR{$t=1$ {\bfseries to} $T$}
       \STATE $\boldsymbol{\theta}_{0}^{t} \leftarrow \tilde{\boldsymbol{\theta}}_{t-1}$ \qquad \qquad \qquad \qquad \qquad \qquad \qquad \qquad \qquad \quad   \textit{; broadcast to all agents} 
       \FOR{$k=1$ {\bfseries to} $K$}
           \STATE Sample $B_t$ trajectories $\{\tau_{t,i}^{(k)}\}_{i=1}^{B_t}$ from $p(\cdot|\boldsymbol{\theta}_0^t)$
           \STATE {$\mu_{t}^{(k)} \triangleq \left\{\begin{array}{ll}
            \frac{1}{B_t} \sum_{i=1}^{B_t}  g(\tau_{t, i}^{(k)} | \boldsymbol{\theta}_{0}^{t} ) & \text { for } k \in \mathcal{G} \\
            * & \text { for } k \notin \mathcal{G}
            \end{array}\right.$ \quad \qquad \qquad   \textit{; push $\mu_t^{(k)}$ to server}}
       \ENDFOR
       \STATE $\mu_t \leftarrow \operatorname{\textbf{FedPG-Aggregate}}(\{\mu_t^{(k)}\}_{k=1}^K)$
       \STATE Sample $N_t \sim Geom(\frac{B_t}{B_t+b_t})$
       \FOR{$n=0$ {\bfseries to} $N_t-1$}
           \STATE Sample $b_t$ trajectories $\{\tau_{n,j}^t\}_{j=1}^{b_t}$ from $p(\cdot|\boldsymbol{\theta}_n^t)$
           \STATE {$v_{n}^{t} \triangleq \frac{1}{b_t} \sum_{j=1}^{b_t} [g(\tau_{n,j}^{t}|\boldsymbol{\boldsymbol{\theta}}_{n}^{t}) - \omega(\tau_{n,j}^t|\boldsymbol{\theta}_{n}^{t}, \boldsymbol{\theta}_{0}^{t}) g(\tau_{n,j}^{t}|\boldsymbol{\theta}_{0}^{t})] + \mu_t$}
           \STATE $\boldsymbol{\theta}_{n+1}^t = \boldsymbol{\theta}_{n}^{t} + \eta_{t}v_{n}^{t}$
       \ENDFOR
       \STATE $\tilde{\boldsymbol{\theta}}_{t} \leftarrow \boldsymbol{\theta}_{N_t}^{t}$
   \ENDFOR
   \STATE {\bfseries Output: $\tilde{\boldsymbol{\theta}}_{a}$} uniformly randomly picked from $\{\tilde{\boldsymbol{\theta}}_t\}_{t=1}^T$ 
\end{algorithmic}
\end{algorithm}

The aggregation is then followed by the SCSG inner loop~\citep{lei2017scsg2} with $N_t$ steps,
where $N_t$ is sampled from a geometric distribution with parameter $\frac{B_t}{B_t + b_t}$ (line 8). 
At step $n$, the server \emph{independently} samples $b_t$ ($b_t \ll B_t$) trajectories $\{\tau_{n,j}^t\}_{j=1}^{b_t}$ using its current policy $\boldsymbol{\theta}_n^t$ (line 10), 
and then updates the policy parameter $\boldsymbol{\theta}_n^t$ based on the following semi-stochastic gradient (lines 11 and 12):
\begin{align*}
    v_{n}^{t} \triangleq \frac{1}{b_t} \sum_{j=1}^{b_t} \left[g(\tau_{n,j}^{t}|\boldsymbol{\theta}_{n}^{t}) - \omega(\tau_{n,j}^t|\boldsymbol{\theta}_{n}^{t}, \boldsymbol{\theta}_{0}^{t}) g(\tau_{n,j}^{t}|\boldsymbol{\theta}_{0}^{t})\right] + \mu_t. \stepcounter{equation}\tag{\theequation}\label{def-semi-stochastic-gradient} 
\end{align*}
\begin{algorithm}[t]
   \renewcommand\thealgorithm{1.1}
   \caption{\textbf{FedPG-Aggregate}}
   \label{alg:FedAgg-BR}
\begin{algorithmic}[1]
  \STATE {\bfseries Input:} Gradient estimates from $K$ agents in round $t$: $\{\mu_t^{(k)}\}_{k=1}^K$, variance bound $\sigma$, filtering threshold $\mathfrak{T}_{\mu}\triangleq2 \sigma \sqrt{\frac{V}{B_t}}$, where $V\triangleq2\operatorname{log}(\frac{2K}{\delta})$ and $\delta \in (0,1)$
        \STATE {$S_1 \triangleq  \{\mu_{t}^{(k)}\} \text { where } k \in[K] \text { s.t. } 
         \left|\left\{k^{\prime}\in[K]:\left\|\mu_{t}^{(k^{\prime})}-\mu_{t}^{(k)}\right\| \leq \mathfrak{T}_{\mu}\right\}\right|>\frac{K}{2}$}
        \STATE {$\mu_{t}^{\text {mom }} \leftarrow \operatorname{argmin}_{\mu_t^{(\tilde{k})}}\|\mu_t^{(\tilde{k})} - \operatorname{mean}(S_1)\| \text{ where } \tilde{k} \in S_1 $}
        \STATE \textit{R1:} $\mathcal{G}_{t} \triangleq \left\{k \in[K]:\left\|\mu_{t}^{(k)}-\mu_{t}^{\text{mom}}\right\| \leq \mathfrak{T}_{\mu}\right\}$
      
        \IF{$\left|\mathcal{G}_{t}\right|<(1-\alpha) K$ \quad} 
            \STATE {$S_2 \triangleq  \{\mu_{t}^{(k)}\} \text { where } k \in[K] \text { s.t. } 
             \left|\left\{k^{\prime}\in[K]:\left\|\mu_{t}^{(k^{\prime})}-\mu_{t}^{(k)}\right\| \leq 2 \sigma\right\}\right|>\frac{K}{2}$}
            \STATE {$\mu_{t}^{\text {mom }} \leftarrow \operatorname{argmin}_{\mu_t^{(\tilde{k})}}\|\mu_t^{(\tilde{k})} - \operatorname{mean}(S_2)\| \text{ where } \tilde{k} \in S_2 $}
            \STATE \textit{R2:} $\mathcal{G}_{t} \triangleq \left\{k \in[K]:\left\|\mu_{t}^{(k)}-\mu_{t}^{\text{mom}}\right\| \leq 2 \sigma\right\}$
        \ENDIF
      \STATE {\bfseries Return:} $\mu_{t} \triangleq \frac{1}{\left|\mathcal{G}_{t}\right|} \sum_{k \in \mathcal{G}_{t}} \mu_{t}^{(k)}$ 
\end{algorithmic}
\end{algorithm}
The last two terms serve as a correction to the gradient estimate to reduce variance and improve the convergence rate of Algorithm~\ref{alg:FedPG-BR}. 
Of note, the semi-stochastic gradient above~\eqref{def-semi-stochastic-gradient} differs from that used in SCSG due to the additional term of $\omega(\tau | \boldsymbol{\theta}_{n}^{t}, \boldsymbol{\theta}_0^t) \triangleq p(\tau | \boldsymbol{\theta}_0^t) / p(\tau | \boldsymbol{\theta}_{n}^{t})$. 
This term is known as the \textit{importance weight} from $p(\tau | \boldsymbol{\theta}_n^t)$ to $p(\tau | \boldsymbol{\theta}_0^t)$ to account for the aforementioned non-stationarity of the distribution in RL \citep{papini2018stochastic,xu2020improvedUAI}. In particular, directly computing $g(\tau_{n,j}^{t}|\boldsymbol{\theta}_{0}^{t})$ results in a biased estimation because the trajectories $\{\tau_{n,j}^t\}_{j=1}^{b_t}$ are sampled from the policy $\boldsymbol{\theta}_n^t$ instead of $\boldsymbol{\theta}_0^t$. 
We prove in Lemma \ref{lemma-unbiased-IS} (Appendix~\ref{app:some_useful_lemmas}) that this importance weight results in an unbiased estimate of the gradient, i.e., $\mathbb{E}_{\tau \sim p\left(\cdot \mid \boldsymbol{\theta}_{n}\right)} [\omega(\tau|\boldsymbol{\theta}_n,\boldsymbol{\theta}_0)g(\tau|\boldsymbol{\theta}_0)] =\nabla J({\boldsymbol{\theta}_0})$.
%


Here we describe the details of our Byzantine filtering step (i.e., the subroutine \textbf{FedPG-Aggregate} in Algorithm \ref{alg:FedAgg-BR}), which is inspired by the works of \citet{alistarh2018byzantine} and \citet{khanduri2019byzantine} in distributed supervised learning.
In any round $t$, we use $\mathcal{G}$ to denote the set of true good agents and use $\mathcal{G}_t$ to denote the set of agents that are believed to be good by the server.
Our Byzantine filtering consists of two filtering rules denoted by R1 (lines 2-4) and R2 (lines 6-8).
R2 is more intuitive to understand, so we start by introducing R2.
Firstly, in line 6, the server constructs a set $S_2$ of \textit{vector medians}~\citep{alistarh2018byzantine} where each element of $S_2$ is chosen from $\{\mu_t^{(k)}\}_{k=1}^K$ if it is close (within $2\sigma$ in Euclidean distance) to more than $K/2$ elements.
Next, the server finds a \textit{Mean of Median} vector $\mu_{t}^{\text{mom}}$ from $S_2$,
which is defined as any $\mu_t^{(\tilde{k})} \in S_2$ that is the closet to the mean of the vectors in $S_2$. 
After $\mu_t^{\text{mom}}$ is selected, the server can construct the set $\mathcal{G}_t$ by filtering out any $\mu_t^{(k)}$ whose distance to $\mu_t^{\text{mom}}$ is larger than $2\sigma$ (line 8). 
This filtering rule is designed based on Assumption~\ref{assumption-bounded-variance} which implies that the maximum distance between any two good agents is $2\sigma$, and our assumption that at least half of the agents are good (i.e., $\alpha < 0.5$). 
We show in Appendix~\ref{proof-claims-in-filtering-strategy} that under these two assumptions,
R2 guarantees that all good agents are included in $\mathcal{G}_t$ (i.e., $|\mathcal{G}_t| \geq (1-\alpha)K$).
We provide a graphical illustration (Fig.~\ref{fig:filtering-diagram} in Appendix~\ref{proof-claims-in-filtering-strategy}) on that if any Byzantine agent is included in $\mathcal{G}_t$, its distance to the true gradient $\nabla J(\boldsymbol{\theta}^t_0)$ is at most $3\sigma$, which ensures that its impact on the algorithm is limited. {Note that the pairwise computation among the weights of all the agents can be implemented using the Euclidean Distance Matrix Trick~\citep{albanie2019euclidean}}.

R1 (lines 2-4) is designed in a similar way: R1 ensures that all good agents are \textit{highly likely} to be included in $\mathcal{G}_t$ by exploiting Lemma~\ref{lemma-martingale} (Appendix~\ref{app:some_useful_lemmas}) to guarantee that \textit{with high probability}, all good agents are \textit{concentrated in a smaller region}. 
That is, define $V\triangleq2\operatorname{log}(2K/\delta)$ and $\delta \in (0,1)$, then with probability of $\geq 1-\delta$, the maximum distance between any two good agents is $\mathfrak{T}_{\mu} \triangleq 2 \sigma \sqrt{V/B_t}$.
Having all good agents in a smaller region improves the filtering strategy, because it makes the Byzantine agents less likely to be selected and reduces their impact even if they are selected.
Therefore, R1 is applied first such that if R1 fails to include all good agents in $\mathcal{G}_t$ (line 4) which happens with probability $<\delta$, R2 is then employed as a backup to ensure that $\mathcal{G}_t$ always include all good agents. Therefore, these two filtering rules ensure in any round $t$ that (a) gradients from good agents are never filtered out, and that (b) if gradients from Byzantine agents are not filtered out, their impact is limited since their maximum distance to $\nabla J(\boldsymbol{\theta}^t_0)$ is bounded by $3\sigma$.

\section{Theoretical results}\label{section:theoretical-results}
Here, we firstly put in place a few assumptions required for our theoretical analysis, all of which are common in the literature. 
\begin{assumption}[On policy derivatives]
    \label{assumption-policy-derivatives}
    Let $\pi_{\boldsymbol{\theta}}(a|s)$ be the policy of an agent at state $s$. There exist constants $G, M > 0$ s.t.\ the log-density of the policy function satisfies, for all $a \in \mathcal{A}$ and $s \in \mathcal{S}$
    \begin{align*}
        |\nabla_{\boldsymbol{\theta}}\operatorname{log} \pi_{\boldsymbol{\theta}}(a | s) | \leq G, \qquad \|\nabla_{\boldsymbol{\theta}}^2\operatorname{log} \pi_{\boldsymbol{\theta}}(a | s) \| \leq M, \qquad
        \forall a \in \mathcal{A}, \forall s \in \mathcal{S}
    \end{align*}
\end{assumption}
Assumption~\ref{assumption-policy-derivatives} 
provides the basis for the smoothness assumption on the objective function $J(\boldsymbol{\theta})$ commonly used in non-convex optimization~\citep{reddi2016svrg3,allen2016svrg2} and also appears in ~\citet{papini2018stochastic,xu2020improvedUAI}.
Specifically, Assumption~\ref{assumption-policy-derivatives} implies:
\begin{proposition}[On function smoothness]
    \label{proposition-uai-grad}
    Under Assumption \ref{assumption-policy-derivatives}, $J(\boldsymbol{\theta})$ is $L$-smooth with $L \triangleq HR(M+HG^2)/(1-\gamma)$. Let $g(\tau | \boldsymbol{\theta})$ be the REINFORCE or GPOMDP gradient estimators. Then for all $ \boldsymbol{\theta}, \boldsymbol{\theta}_1, \boldsymbol{\theta}_2 \in \mathbb{R}^d$, it holds that
    \begin{align*}
        \|g(\tau|\boldsymbol{\theta})\| \leq C_g, \qquad \left\|g\left(\tau \mid \boldsymbol{\theta}_{1}\right)-g\left(\tau \mid \boldsymbol{\theta}_{2}\right)\right\| \leq L_{g}\left\|\boldsymbol{\theta}_{1}-\boldsymbol{\theta}_{2}\right\|
    \end{align*}
    where $L_g \triangleq HM(R+|C_b|)/(1-\gamma), C_g \triangleq HG(R+|C_b|)/(1-\gamma)$ and $C_b$ is the baseline reward.
\end{proposition}
Proposition~\ref{proposition-uai-grad} is important for deriving a fast convergence rate and its proof can be found in~\citet{xu2020improvedUAI}.
Next, we need an assumption on the variance of the importance weights (Section~\ref{section:settings}).
\begin{assumption}[On variance of the importance weights]
    \label{assumption-uai-weight}
    There exists a constant $W < \infty$ such that for each policy pairs in Algorithm \ref{alg:FedPG-BR}, it holds
    \begin{align*}
        \operatorname{Var}(\omega(\tau|\boldsymbol{\theta}_1, \boldsymbol{\theta}_2)) \leq W, \qquad
        \forall \boldsymbol{\theta}_1, \boldsymbol{\theta}_2 \in \mathbb{R}^d, \tau \sim p(\cdot | \boldsymbol{\theta}_1)
    \end{align*}
\end{assumption}
Assumption~\ref{assumption-uai-weight} has also been made by \citet{papini2018stochastic, xu2020improvedUAI}. Now we present the convergence guarantees for our FedPG-BR algorithm:
\begin{theorem}[Convergence of FedPG-BR]
\label{main-theorem}
Assume uniform initial state distribution across agents, and the gradient estimator is set to be the REINFORCE or GPOMDP estimator. Under Assumptions \ref{assumption-bounded-variance}, \ref{assumption-policy-derivatives}, and \ref{assumption-uai-weight}, if we choose $\eta_t \leq \frac{1}{2 \Psi B_t^{2/3}}$, $b_t = 1$, and $B_t = B \geq 4\Phi L^{-2}$ where $\Phi \triangleq L_g + C_g^2C_w$, $\Psi \triangleq (L(L_g + C_g^2C_w))^{1/3}$, 
$L, L_g, C_g$ are defined in Proposition \ref{proposition-uai-grad} and $C_w$ is defined in Lemma \ref{lemma-uai-variance-IS}, $\delta \in (0,1)$ such that $e^{\frac{\delta B_t}{2(1-2 \delta)}} \leq \frac{2 K}{\delta} \leq e^{\frac{B_t}{2}}$ and $\delta \leq \frac{1}{5KB_t}$,
then the output $\tilde{\boldsymbol{\theta}}_a$ of Algorithm \ref{alg:FedPG-BR} satisfies
\begin{align*}
    \mathbb{E}[\|\nabla J(\tilde{\boldsymbol{\theta}}_{a})\|^{2}] \leq {\frac{2 \Psi \left[J(\tilde{\boldsymbol{\theta}}^*)-J(\tilde{\boldsymbol{\theta}}_{0})\right]}{TB^{1 / 3}}} +\frac{8 \sigma^{2}}{(1-\alpha)^{2} K B}+\frac{96 \alpha^{2} \sigma^{2} V}{(1-\alpha)^{2} B}
\end{align*}
where $0 \leq \alpha <0.5$ and $\tilde{\boldsymbol{\theta}}^*$ is a global maximizer of $J$.
\end{theorem}
This theorem leads to many interesting insights. When $K = 1, \alpha = 0$, Theorem \ref{main-theorem} reduces to $\mathbb{E}\|\nabla J (\tilde{\boldsymbol{\theta}}_a)\|^2 \leq 2\Psi [J(\tilde{\boldsymbol{\theta}}^*) - J(\tilde{\boldsymbol{\theta}}_0)]/TB^{1/3} + 8\sigma^2 / B$. The second term here $O(1/B)$, which also shows up in SVRPG~\citep{papini2018stochastic, xu2020improvedUAI}, results from the full gradient approximation in Equation~\eqref{def-gradient-estimate} in each round. 
In this case, our theorem implies that $\mathbb{E}\|\nabla J (\tilde{\boldsymbol{\theta}}_a)\|^2 = O(\Psi [J(\tilde{\boldsymbol{\theta}}^*) - J(\tilde{\boldsymbol{\theta}}_0)]/TB^{1/3})$ which is consistent with SCSG for $L$-smooth non-convex objective functions~\citep{lei2017scsg2}. 
Moreover, using $\mathbb{E}[Traj(\epsilon)]$ to denote the expected number of trajectories required by each agent to achieve $\mathbb{E}[\|\nabla J(\tilde{\boldsymbol{\theta}}_a)\|^2] \leq \epsilon$, Theorem~\ref{main-theorem} leads to:
\begin{corollary}[Sample complexity of FedPG-BR]
\label{main-corollary}
Under the same assumptions as Theorem~\ref{main-theorem}, let $\epsilon > 0$, we have: (i) $\mathbb{E}[Traj(\epsilon)] = {O}(\frac{1}{\epsilon^{5/3}K^{2/3}} + \frac{\alpha^{4/3}}{\epsilon^{5/3}})$; (ii) When $\alpha = 0$, we have $\mathbb{E}[Traj(\epsilon)] = {O}(\frac{1}{\epsilon^{5/3}K^{2/3}})$;
(iii) When $K\!=\!1$, we have  $\mathbb{E}[Traj(\epsilon)] = {O}(\frac{1}{\epsilon^{5/3}})$
\end{corollary}
\begin{table}[h]
    \centering
    \caption{Sample complexities of relevant works to achieve $\mathbb{E}\|\nabla J(\boldsymbol{\theta})\|^2 \leq\epsilon$.}\label{tab:table-bound}
    \begin{tabular}{cll}
        \toprule
        \multicolumn{1}{c} { \textbf{SETTINGS} } & \textbf{METHODS} & \textbf{COMPLEXITY}\\
        \midrule
        \multirow{5}{*}{$K=1$} & REINFORCE \citep{williams1992REINFORCE} & $O(1 / \epsilon^{2})$ \\
                             & GPOMDP \citep{baxter2001GPOMDP} & $O(1 / \epsilon^{2})$ \\
                             & SVRPG \citep{papini2018stochastic} & $O(1 / \epsilon^{2})$ \\
                             & SVRPG \citep{xu2020improvedUAI} & $O(1 / \epsilon^{5 / 3})$ \\
                             & \textbf{FedPG-BR} & $O(1 / \epsilon^{5 / 3})$\\
        \midrule
        $K>1, \alpha=0$      & \textbf{FedPG-BR} & $O(\frac{1}{\epsilon^{5/3}K^{2/3}})$\\
        $K>1, \alpha>0$      & \textbf{FedPG-BR} & $O(\frac{1}{\epsilon^{5/3}K^{2/3}} + \frac{\alpha^{4/3}}{\epsilon^{5/3}})$ \\ %
        \bottomrule
    \end{tabular} 
\end{table}
%

We present a straightforward comparison of the sample complexity of related works in Table~\ref{tab:table-bound}. 
Both REINFORCE and GPOMDP have a sample complexity of $O(1/\epsilon^2)$ since they use stochastic gradient-based optimization.
\citet{xu2020improvedUAI} has made a refined analysis of SVRPG to improve its sample complexity from $O(1/\epsilon^{2})$~\citep{papini2018stochastic} to $O(1/\epsilon^{5/3})$.
Corollary \ref{main-corollary} \textit{(iii)} reveals that
the sample complexity of FedPG-BR in the single-agent setup agrees with that of SVRPG derived by~\citet{xu2020improvedUAI}.

When $K > 1$, $\alpha = 0$, Corollary \ref{main-corollary} \textit{(ii)}
implies that the total number of trajectories required by each agent is upper-bounded by $O(1/(\epsilon^{5/3}K^{2/3}))$. This result gives us the theoretical grounds to encourage more agents to participate in the federation, since the number of trajectories each agent needs to sample decays at a rate of $O(1/K^{2/3})$. 
This guaranteed improvement in sample efficiency is highly desirable in practical systems with a large number of agents.
%


%

Next, for a more realistic system where an $\alpha$-fraction $(\alpha > 0)$ of the agents are Byzantine agents, 
Corollary~\ref{main-corollary} \textit{(i)} assures us that the total number of trajectories required by each agent will be increased by only an additive term of $O(\alpha ^ {4/3}/\epsilon^{5/3})$. 
This term is unavoidable due to the presence of Byzantine agents in FRL systems. 
However, the bound implies that the impact of Byzantine agents on the overall convergence is limited, which aligns with the discussions on our Byzantine filtering strategy (Section~\ref{subsection:algorithm-description}),
and will be empirically verified in our experiments. 
Moreover, the impact of Byzantine agents on the convergence vanishes when $\alpha \rightarrow 0$. 
That is, when the system is ideal ($\alpha=0$), our Byzantine filtering step induces no effect on the convergence.

\section{Experiments}\label{section:experiments}
We evaluate the empirical performances of FedPG-BR with and without Byzantine agents on different RL benchmarks, including 
CartPole balancing~\citep{barto1983CartPole}, LunarLander,
and the 3D continuous locomotion control task of Half-Cheetah \citep{duan2016TasksOnMujoco}. 
In all experiments, we measure the performance online such that in each iteration, we evaluate the current policy of the server by using it to 
independently interact with the test MDP for 10 trajectories
and reporting the mean returns.
Each experiment is independently repeated 10 times with different random seeds and policy initializations, both of which are shared among all algorithms for fair comparisons. 
The results are averaged over the 10 independent runs with $90\%$ bootstrap confidence intervals. 
Due to space constraints, some experimental details are deferred to Appendix \ref{EXP-Settings-appendix}.
%
%

\textbf{Performances in ideal systems with $\alpha=0$. }
We firstly evaluate the performances of FedPG-BR in ideal systems with $\alpha = 0$, i.e., no Byzantine agents.
We compare FedPG-BR ($K\!=\!1,3,10$) with vanilla policy gradient using GPOMDP\footnote{Since GPOMDP has been repeatedly found to be comparable to or better than REINFORCE~\citep{papini2018stochastic,xu2020improvedUAI}.
} and SVRPG.
The results in all three tasks are plotted in Fig.~\ref{fig:exp1}. The figures show that FedPG-BR ($K\!=\!1$) and SVRPG perform comparably, both outperforming GPOMDP.
This aligns with the results in Table~\ref{tab:table-bound} showing that FedPG-BR ($K\!=\!1$) and SVRPG share the same sample complexity, and both are provably more sample-efficient than GPOMDP.
Moreover, the performance of FedPG-BR is improved significantly with the federated of only $K=3$ agents, and improved even further when $K=10$.
This corroborates our theoretical insights implying that the federation of more agents (i.e., larger $K$) improves the sample efficiency of FedPG-BR (Section~\ref{section:theoretical-results}), 
and verifies the practical performance benefit offered by the participation of more agents.
\begin{figure}[h]
    \centering
    \includegraphics[height=1.2in] {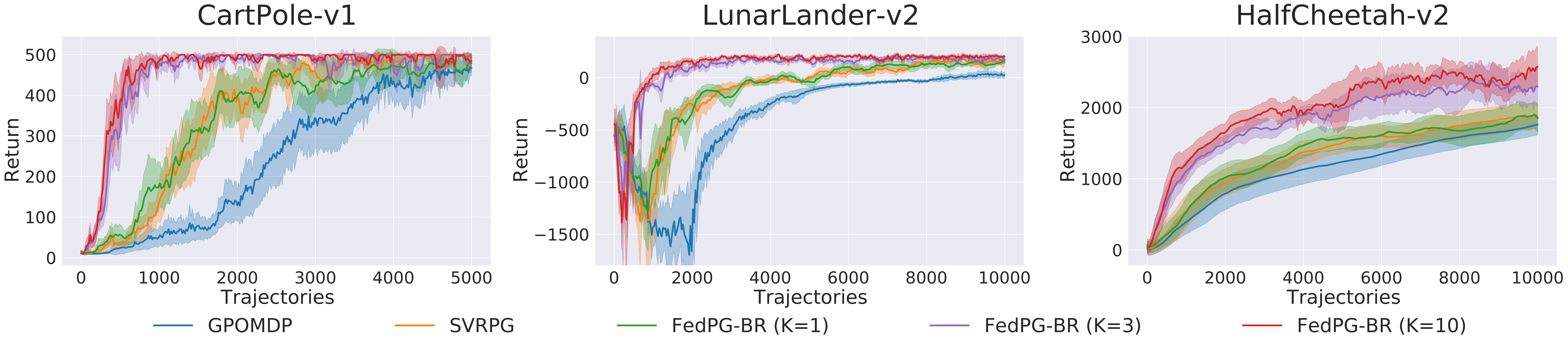}
    \caption{Performance of FedPG-BR in ideal systems with $\alpha=0$ for the three tasks.}
    \label{fig:exp1}
\end{figure}
\begin{figure}[b]
    \centering
    \includegraphics[height=1.3in] {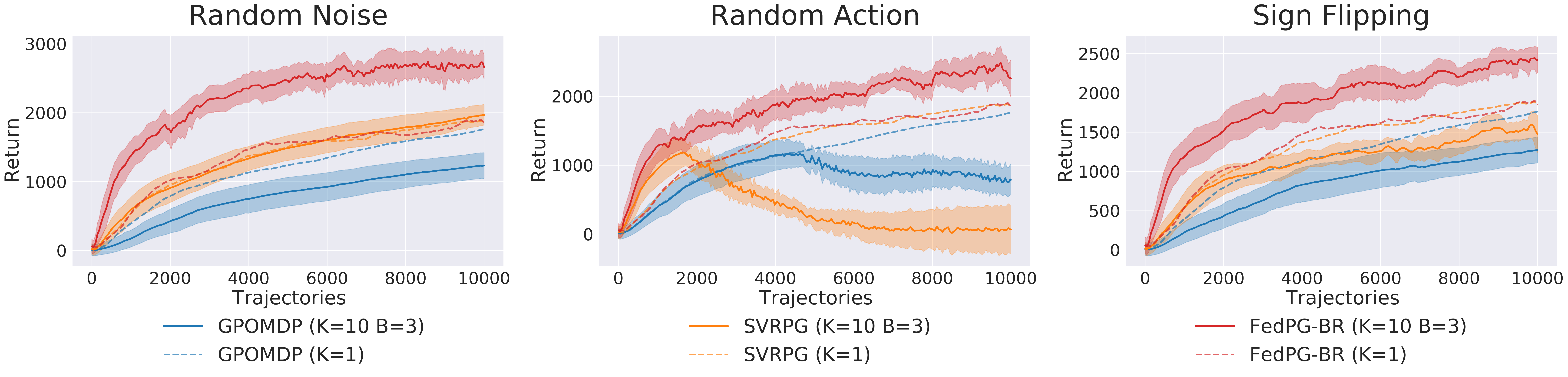}
    \caption{Performance of FedPG-BR in practical systems with $\alpha>0$ for HalfCheetah. 
    Each subplot corresponds to a different type of Byzantine failure exercised by the 3 Byzantine agents.}
    \label{fig:exp2}
\end{figure}
\begin{figure}[!b]
    \centering
    \includegraphics[height=1.2in] {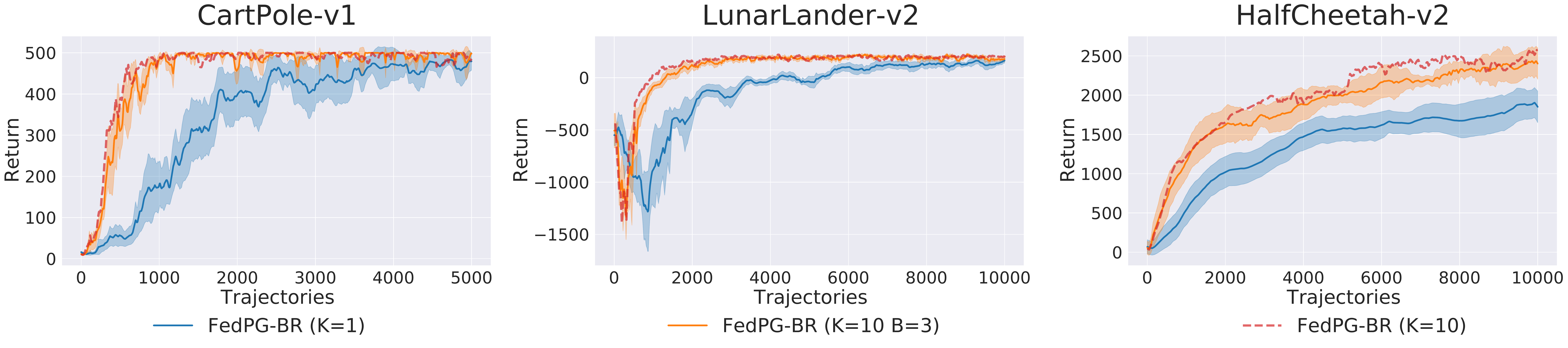}
    \caption{Performance of FedPG-BR in practical systems 
    against FedPG attack.
    }
    \label{fig:exp3}
\end{figure}

\textbf{Performances in practical systems with $\alpha>0$. }
Next, we investigate the impact of Byzantine agents (i.e., random failures or adversarial attacks) on the sample efficiency, which is critical for the practical deployment of FRL algorithms.
In this experiment, we use $K=10$ agents among which $3$ are Byzantine agents, and we simulate different types of Byzantine failures:
(a) \textit{Random Noise (RN)}: each Byzantine agent sends 
a random vector to the server;
(b) \textit{Random Action (RA)}: every Byzantine agent ignores the policy from the server and takes actions randomly, which is used to simulate random system failures (e.g., hardware failures) and results in false gradient computations since the trajectories are no longer sampled according to the policy;
(c) \textit{Sign Filliping (SF)}: each Byzantine agent computes the correct gradient but sends the scaled negative gradient (multiplied by $-2.5$),
which is used to simulate adversarial attacks aiming to manipulate the direction of policy update at the server.

For comparison, we have adapted both GPOMDP and SVRPG to the FRL setting (pseudocode is provided in Appendix \ref{algo-fedVPG-fedSVRPG}). 
Fig.~\ref{fig:exp2} shows the results using the HalfCheetah task.
We have also included the performances of GPOMDP, SVRPG and FedPG-BR in the single-agent settting ($K\!=\!1$) as dotted-line (mean value of 10 independent runs) for reference.
The figures show that for both GPOMDP and SVRPG, the 3 Byzantine agents cause the performance of their federated versions to be worse than that in the single-agent setting.
Particularly, RA agents (middle figure) render GPOMDP and SVRPG unlearnable, i.e., unable to converge at all.
In contrast, our FedPG-BR is robust against all three types of Byzantine failures. That is, FedPG-BR ($K\!=\!10\ B\!=\!3$) with 3 Byzantine agents still significantly outperforms the single-agent setting, and more importantly, \emph{performs comparably to FedPG-BR ($K\!=\!10$) with 10 good agents}.
This is because our Byzantine filtering strategy can effectively filter out those Byzantine agents.
These results demonstrate that even in practical systems which are subject to random failures or adversarial attacks, FedPG-BR is still able to deliver superior performances.
This provides an assurance on the reliability of our FedPG-BR algorithm to promote its practical deployment, and significantly improves the practicality of FRL.
The results for the CartPole and LunarLander tasks, which yield the same insights as discussed here, can be found in Appendix~\ref{additional-experiments}.
%

\textbf{Performance of FedPG-BR against FedPG attack. }
We have discussed (Section~\ref{subsection:algorithm-description}) and shown through theoretical analysis (Section~\ref{section:theoretical-results}) that even when our Byzantine filtering strategy fails, the impact of the Byzantine agents on the performance of our algorithm is still limited. Here we verify this empirically.
To this end, we design a new type of Byzantine agents who have perfect knowledge about our Byzantine filtering strategy, and call it \textit{FedPG attacker}. The goal of FedPG attackers are to collude with each other to attack our algorithm without being filtered out.
To achieve this, FedPG attackers
firstly estimate $\nabla J(\boldsymbol{\theta}_0^t)$ using the mean of their gradients $\bar{\mu}_t$, and estimate $\sigma$ by calculating the maximum Euclidean distance between the gradients of any two FedPG attackers as $2\bar{\sigma}$.
Next, all FedPG attackers send the vector $\bar{\mu}_t + 3\bar{\sigma}$ to the server.
Recall we have discussed in Section~\ref{subsection:algorithm-description} that if a Byzantine agent is not filtered out, its distance to the true gradient $\nabla J(\boldsymbol{\theta}^t_0)$ is at most $3\sigma$.
Therefore, if $\bar{\mu}_t$ and $\bar{\sigma}$ are estimated accurately, the vectors from the FedPG attackers can exert negative impact on the convergence while still evading our Byzantine filtering.

We again use $K=10$ agents, among which 3 are FedPG attackers.
The results (Fig.~\ref{fig:exp3}) show that in all three tasks, even against such stong attackers with perfect knowledge of our Byzantine filtering strategy, FedPG-BR ($K\!=\!10\ B\!=\!3$) still manages to significantly outperform FedPG-BR ($K\!=\!1$) in the single-agent setting.
Moreover, the performance of FedPG-BR ($K\!=\!10\ B\!=\!3$) is only marginally worsened compared with FedPG-BR ($K\!=\!10$) with 10 good agents.
This corroborates our theoretical analysis showing that
although we place no assumption on the gradients sent by the Byzantine agents,
they only contribute an additive term of $O(\alpha^{4/3}/\epsilon^{5/3})$ to the sample complexity (Section~\ref{section:theoretical-results}).
These results demonstrate the empirical robustness of FedPG-BR even against strong attackers, hence further highlighting its practical reliability.

\section{Conclusion and future work}\label{section:conclusion}
Federation is promising in boosting the sample efficiency of RL agents, without sharing their trajectories. Due to the high sampling cost of RL, the design of FRL systems appeals for theoretical guarantee on its convergence which is, however, vulnerable to failures and attacks in practical setup, as demonstrated. This paper provides the theoretical ground to study the sample efficiency of FRL with respect to the number of participating agents, while accounting for faulty agents. We verify the empirical efficacy of the proposed FRL framework in systems with and without different types of faulty agents on various RL benchmarks.  

Variance control is the key to exploiting Assumption~\ref{assumption-bounded-variance} on the bounded variance of PG estimators in our filter design.
As a result, our framework is restricted to the variance-reduced policy gradient methods. Intuitively, it is worth studying the fault-tolerant federation of other policy optimization methods. Another limitation of this work is that agents are assumed to be homogeneous, while in many real-world scenarios, RL agents are heterogeneous. Therefore, it would be interesting to explore the possibility of heterogeneity of agents in fault-tolerant FRL in future works.
Moreover, another interesting future work is to apply our Byzantine filtering strategy to other federated sequential decision-making problems such as federated bandit~\cite{dubey2020differentially,shi2021federated,zhu2021federated} and federated/collaborative Bayesian optimization~\cite{dai2020federated,dai2021differentially,sim2021collaborative}, as well as other settings of collaborative multi-party ML~\cite{xu2021gradient,xu2021validation,hoang2021aid,lam2021model,sim2020collaborative,hoang2020learning,ouyang2020gaussian,hoang2019collective,hoang2019collective2,ouyang2018gaussian,chen2015gaussian,chen2013gaussian,chen2012decentralized}, to equip them with theoretically guaranteed fault tolerance.

\begin{ack}
This research/project is supported by A*STAR under its RIE$2020$ Advanced Manufacturing and Engineering (AME) Industry Alignment Fund – Pre Positioning (IAF-PP) (Award A$19$E$4$a$0101$) and
its A*STAR Computing and Information Science Scholarship (ACIS) awarded to Flint Xiaofeng Fan. Wei Jing is supported by Alibaba Innovative Research (AIR) Program.
\end{ack}

\medskip

\bibliography{references}

\begin{thebibliography}{77}
\providecommand{\natexlab}[1]{#1}
\providecommand{\url}[1]{\texttt{#1}}
\expandafter\ifx\csname urlstyle\endcsname\relax
  \providecommand{\doi}[1]{doi: #1}\else
  \providecommand{\doi}{doi: \begingroup \urlstyle{rm}\Url}\fi

\bibitem[Mnih et~al.(2013)Mnih, Kavukcuoglu, Silver, Graves, Antonoglou,
  Wierstra, and Riedmiller]{mnih2013playing}
Volodymyr Mnih, Koray Kavukcuoglu, David Silver, Alex Graves, Ioannis
  Antonoglou, Daan Wierstra, and Martin Riedmiller.
\newblock Playing atari with deep reinforcement learning.
\newblock {arXiv}:1312.5602, 2013.

\bibitem[Sergey et~al.(2015)Sergey, Wagener, and Abbeel]{sergey2015learning}
Levine Sergey, Nolan Wagener, and Pieter Abbeel.
\newblock Learning contact-rich manipulation skills with guided policy search.
\newblock In \emph{Proceedings of the 2015 IEEE International Conference on
  Robotics and Automation (ICRA), Seattle, WA, USA}, pages 26--30, 2015.

\bibitem[Liu et~al.(2020)Liu, See, Ngiam, Celi, Sun, and Feng]{RL-MIMIC-survey}
Siqi Liu, Kay~Choong See, Kee~Yuan Ngiam, Leo~Anthony Celi, Xingzhi Sun, and
  Mengling Feng.
\newblock Reinforcement learning for clinical decision support in critical
  care: comprehensive review.
\newblock \emph{Journal of Medical Internet Research}, 22\penalty0
  (7):\penalty0 e18477, 2020.

\bibitem[Dulac-Arnold et~al.(2019)Dulac-Arnold, Mankowitz, and
  Hester]{dulac2019challengesRealWorld-RL}
Gabriel Dulac-Arnold, Daniel Mankowitz, and Todd Hester.
\newblock Challenges of real-world reinforcement learning.
\newblock {arXiv}:1904.12901, 2019.

\bibitem[Levine et~al.(2020)Levine, Kumar, Tucker, and Fu]{levine2020offlineRL}
Sergey Levine, Aviral Kumar, George Tucker, and Justin Fu.
\newblock Offline reinforcement learning: Tutorial, review, and perspectives on
  open problems.
\newblock {arXiv}:2005.01643, 2020.

\bibitem[Komorowski et~al.(2018)Komorowski, Celi, Badawi, Gordon, and
  Faisal]{Nature-RL-MIMIC}
Matthieu Komorowski, Leo~A Celi, Omar Badawi, Anthony~C Gordon, and A~Aldo
  Faisal.
\newblock The artificial intelligence clinician learns optimal treatment
  strategies for sepsis in intensive care.
\newblock \emph{Nature medicine}, 24\penalty0 (11):\penalty0 1716--1720, 2018.

\bibitem[Lin et~al.(2018)Lin, Stanley, Ghassemi, and Nemati]{DDPG-RL-MIMIC}
Rongmei Lin, Matthew~D Stanley, Mohammad~M Ghassemi, and Shamim Nemati.
\newblock A deep deterministic policy gradient approach to medication dosing
  and surveillance in the icu.
\newblock In \emph{2018 40th Annual International Conference of the IEEE
  Engineering in Medicine and Biology Society (EMBC)}, pages 4927--4931. IEEE,
  2018.

\bibitem[Kone{\v{c}}n{\`y} et~al.(2016)Kone{\v{c}}n{\`y}, McMahan, Ramage, and
  Richt{\'a}rik]{konevcny2016federated}
Jakub Kone{\v{c}}n{\`y}, H~Brendan McMahan, Daniel Ramage, and Peter
  Richt{\'a}rik.
\newblock Federated optimization: Distributed machine learning for on-device
  intelligence.
\newblock {arXiv}:1610.02527, 2016.

\bibitem[Kairouz et~al.(2019)Kairouz, McMahan, Avent, Bellet, Bennis, Bhagoji,
  Bonawitz, Charles, Cormode, Cummings, et~al.]{kairouz2019FedLearn}
Peter Kairouz, H~Brendan McMahan, Brendan Avent, Aur{\'e}lien Bellet, Mehdi
  Bennis, Arjun~Nitin Bhagoji, Keith Bonawitz, Zachary Charles, Graham Cormode,
  Rachel Cummings, et~al.
\newblock Advances and open problems in federated learning.
\newblock {arXiv}:1912.04977, 2019.

\bibitem[Li et~al.(2019)Li, Wen, Wu, Hu, Wang, Li, Liu, and He]{li2019survey}
Qinbin Li, Zeyi Wen, Zhaomin Wu, Sixu Hu, Naibo Wang, Yuan Li, Xu~Liu, and
  Bingsheng He.
\newblock A survey on federated learning systems: vision, hype and reality for
  data privacy and protection.
\newblock {arXiv}:1907.09693, 2019.

\bibitem[Li et~al.(2020{\natexlab{a}})Li, Wen, and He]{li2020practical}
Qinbin Li, Zeyi Wen, and Bingsheng He.
\newblock Practical federated gradient boosting decision trees.
\newblock In \emph{Proceedings of the AAAI Conference on Artificial
  Intelligence}, volume~34, pages 4642--4649, 2020{\natexlab{a}}.

\bibitem[Zhuo et~al.(2019)Zhuo, Feng, Xu, Yang, and
  Lin]{zhuo2019federatedRLYangQ}
Hankz~Hankui Zhuo, Wenfeng Feng, Qian Xu, Qiang Yang, and Yufeng Lin.
\newblock Federated reinforcement learning.
\newblock {arXiv}:1901.08277, 2019.

\bibitem[Liang et~al.(2019)Liang, Liu, Chen, Liu, and
  Yang]{liang2019FedRL-for-AV}
Xinle Liang, Yang Liu, Tianjian Chen, Ming Liu, and Qiang Yang.
\newblock Federated transfer reinforcement learning for autonomous driving.
\newblock {arXiv}:1910.06001, 2019.

\bibitem[Nadiger et~al.(2019)Nadiger, Kumar, and
  Abdelhak]{nadiger2019federatedRL}
Chetan Nadiger, Anil Kumar, and Sherine Abdelhak.
\newblock Federated reinforcement learning for fast personalization.
\newblock In \emph{2019 IEEE Second International Conference on Artificial
  Intelligence and Knowledge Engineering (AIKE)}, pages 123--127. IEEE, 2019.

\bibitem[Lim et~al.(2020)Lim, Kim, Heo, and Han]{lim2020federated-Sensors}
Hyun-Kyo Lim, Ju-Bong Kim, Joo-Seong Heo, and Youn-Hee Han.
\newblock Federated reinforcement learning for training control policies on
  multiple iot devices.
\newblock \emph{Sensors}, 20\penalty0 (5):\penalty0 1359, 2020.

\bibitem[Liu et~al.(2019)Liu, Wang, and Liu]{liu2019FedRL-for-robots}
Boyi Liu, Lujia Wang, and Ming Liu.
\newblock Lifelong federated reinforcement learning: a learning architecture
  for navigation in cloud robotic systems.
\newblock \emph{IEEE Robotics and Automation Letters}, 4\penalty0 (4):\penalty0
  4555--4562, 2019.

\bibitem[Yu et~al.(2020)Yu, Chen, Zhou, Gong, and Wu]{yu2020FedRLfor5G}
Shuai Yu, Xu~Chen, Zhi Zhou, Xiaowen Gong, and Di~Wu.
\newblock When deep reinforcement learning meets federated learning:
  Intelligent multi-timescale resource management for multi-access edge
  computing in 5{G} ultra dense network.
\newblock \emph{IEEE Internet of Things Journal}, 2020.

\bibitem[Papini et~al.(2018)Papini, Binaghi, Canonaco, Pirotta, and
  Restelli]{papini2018stochastic}
Matteo Papini, Damiano Binaghi, Giuseppe Canonaco, Matteo Pirotta, and Marcello
  Restelli.
\newblock Stochastic variance-reduced policy gradient.
\newblock {arXiv}:1806.05618, 2018.

\bibitem[Xu et~al.(2020)Xu, Gao, and Gu]{xu2020improvedUAI}
Pan Xu, Felicia Gao, and Quanquan Gu.
\newblock An improved convergence analysis of stochastic variance-reduced
  policy gradient.
\newblock In \emph{Uncertainty in Artificial Intelligence}, pages 541--551.
  PMLR, 2020.

\bibitem[Cao et~al.(2019)Cao, Xiao, Cyr, Zhou, Park, Rampazzi, Chen, Fu, and
  Mao]{cao2019adversarialAttack-AV}
Yulong Cao, Chaowei Xiao, Benjamin Cyr, Yimeng Zhou, Won Park, Sara Rampazzi,
  Qi~Alfred Chen, Kevin Fu, and Z~Morley Mao.
\newblock Adversarial sensor attack on lidar-based perception in autonomous
  driving.
\newblock In \emph{Proceedings of the 2019 ACM SIGSAC Conference on Computer
  and Communications Security}, pages 2267--2281, 2019.

\bibitem[Lamport et~al.(1982)Lamport, SHOSTAK, and PEASE]{lamport1982byzantine}
LESLIE Lamport, ROBERT SHOSTAK, and MARSHALL PEASE.
\newblock The byzantine generals problem.
\newblock \emph{ACM Transactions on Programming Languages and Systems},
  4\penalty0 (3):\penalty0 382--401, 1982.

\bibitem[Lynch(1996)]{distributed-algorithms-book}
Nancy~A. Lynch.
\newblock \emph{Distributed Algorithms}.
\newblock Morgan Kaufmann Publishers Inc., San Francisco, CA, USA, 1996.
\newblock ISBN 9780080504704.

\bibitem[Castro and Liskov(1999)]{castro1999practicalByzantine}
Miguel Castro and Barbara Liskov.
\newblock Practical byzantine fault tolerance.
\newblock In \emph{OSDI}, volume~99, pages 173--186, 1999.

\bibitem[Li et~al.(2020{\natexlab{b}})Li, Sahu, Talwalkar, and
  Smith]{li2020federatedSurvey2}
Tian Li, Anit~Kumar Sahu, Ameet Talwalkar, and Virginia Smith.
\newblock Federated learning: Challenges, methods, and future directions.
\newblock \emph{IEEE Signal Processing Magazine}, 37\penalty0 (3):\penalty0
  50--60, 2020{\natexlab{b}}.

\bibitem[Cauchy(1847)]{cauchy1847GD}
Augustin Cauchy.
\newblock M{\'e}thode g{\'e}n{\'e}rale pour la r{\'e}solution des systemes
  d’{\'e}quations simultan{\'e}es.
\newblock \emph{Comp. Rend. Sci. Paris}, 25\penalty0 (1847):\penalty0 536--538,
  1847.

\bibitem[Nesterov(2013)]{nesterov2013introductory}
Yurii Nesterov.
\newblock \emph{Introductory lectures on convex optimization: A basic course},
  volume~87.
\newblock Springer Science \& Business Media, 2013.

\bibitem[Robbins and Monro(1951)]{robbins1951stochastic}
Herbert Robbins and Sutton Monro.
\newblock A stochastic approximation method.
\newblock \emph{The annals of mathematical statistics}, pages 400--407, 1951.

\bibitem[Bottou and Cun(2003)]{bottou2003sgd}
L{\'e}on Bottou and Yann Cun.
\newblock Large scale online learning.
\newblock \emph{Advances in neural information processing systems},
  16:\penalty0 217--224, 2003.

\bibitem[Nemirovsky and Yudin(1983)]{nemirovsky1983sgdConvergence}
Arkadi~Semenovich Nemirovsky and David~Borisovich Yudin.
\newblock Problem complexity and method efficiency in optimization.
\newblock 1983.

\bibitem[Johnson and Zhang(2013)]{johnson2013svrg1}
Rie Johnson and Tong Zhang.
\newblock Accelerating stochastic gradient descent using predictive variance
  reduction.
\newblock In \emph{Advances in neural information processing systems}, pages
  315--323, 2013.

\bibitem[Xiao and Zhang(2014)]{xiao2014svrg4}
Lin Xiao and Tong Zhang.
\newblock A proximal stochastic gradient method with progressive variance
  reduction.
\newblock \emph{SIAM Journal on Optimization}, 24\penalty0 (4):\penalty0
  2057--2075, 2014.

\bibitem[Allen-Zhu and Hazan(2016{\natexlab{a}})]{allen2016svrg2}
Zeyuan Allen-Zhu and Elad Hazan.
\newblock Variance reduction for faster non-convex optimization.
\newblock In \emph{International conference on machine learning}, pages
  699--707, 2016{\natexlab{a}}.

\bibitem[Reddi et~al.(2016)Reddi, Hefny, Sra, Poczos, and
  Smola]{reddi2016svrg3}
Sashank~J Reddi, Ahmed Hefny, Suvrit Sra, Barnabas Poczos, and Alex Smola.
\newblock Stochastic variance reduction for nonconvex optimization.
\newblock In \emph{International conference on machine learning}, pages
  314--323, 2016.

\bibitem[Lei and Jordan(2016)]{lei2016scsg1}
Lihua Lei and Michael~I Jordan.
\newblock Less than a single pass: Stochastically controlled stochastic
  gradient method.
\newblock {arXiv}:1609.03261, 2016.

\bibitem[Lei et~al.(2017)Lei, Ju, Chen, and Jordan]{lei2017scsg2}
Lihua Lei, Cheng Ju, Jianbo Chen, and Michael~I Jordan.
\newblock Non-convex finite-sum optimization via scsg methods.
\newblock In \emph{Advances in Neural Information Processing Systems}, pages
  2348--2358, 2017.

\bibitem[Sutton and Barto(2018)]{sutton2018reinforcement}
Richard~S Sutton and Andrew~G Barto.
\newblock \emph{Reinforcement learning: An introduction}.
\newblock MIT press, 2018.

\bibitem[Schulman et~al.(2015)Schulman, Levine, Abbeel, Jordan, and
  Moritz]{schulman2015TRPO}
John Schulman, Sergey Levine, Pieter Abbeel, Michael Jordan, and Philipp
  Moritz.
\newblock Trust region policy optimization.
\newblock In \emph{International conference on machine learning}, pages
  1889--1897, 2015.

\bibitem[Schulman et~al.(2017)Schulman, Wolski, Dhariwal, Radford, and
  Klimov]{schulman2017PPO}
John Schulman, Filip Wolski, Prafulla Dhariwal, Alec Radford, and Oleg Klimov.
\newblock Proximal policy optimization algorithms.
\newblock {arXiv}:1707.06347, 2017.

\bibitem[Williams(1992)]{williams1992REINFORCE}
Ronald~J Williams.
\newblock Simple statistical gradient-following algorithms for connectionist
  reinforcement learning.
\newblock \emph{Machine learning}, 8\penalty0 (3-4):\penalty0 229--256, 1992.

\bibitem[Baxter and Bartlett(2001)]{baxter2001GPOMDP}
Jonathan Baxter and Peter~L Bartlett.
\newblock Infinite-horizon policy-gradient estimation.
\newblock \emph{Journal of Artificial Intelligence Research}, 15:\penalty0
  319--350, 2001.

\bibitem[Du et~al.(2017)Du, Chen, Li, Xiao, and Zhou]{du2017SVRGPolicyEval}
Simon~S Du, Jianshu Chen, Lihong Li, Lin Xiao, and Dengyong Zhou.
\newblock Stochastic variance reduction methods for policy evaluation.
\newblock {arXiv}:1702.07944, 2017.

\bibitem[Peng et~al.(2020)Peng, Touati, Vincent, and Precup]{ijcai2020-0374}
Zilun Peng, Ahmed Touati, Pascal Vincent, and Doina Precup.
\newblock Svrg for policy evaluation with fewer gradient evaluations.
\newblock In \emph{Proceedings of the Twenty-Ninth International Joint
  Conference on Artificial Intelligence, {IJCAI-20}}, pages 2697--2703.
  International Joint Conferences on Artificial Intelligence Organization,
  2020.

\bibitem[Xu et~al.(2017)Xu, Liu, and Peng]{xu2017SVRGTRPO}
Tianbing Xu, Qiang Liu, and Jian Peng.
\newblock Stochastic variance reduction for policy gradient estimation.
\newblock {arXiv}:1710.06034, 2017.

\bibitem[Blanchard et~al.(2017)Blanchard, El~Mhamdi, Guerraoui, and
  Stainer]{blanchard2017machine-Krum}
Peva Blanchard, El~Mahdi El~Mhamdi, Rachid Guerraoui, and Julien Stainer.
\newblock Machine learning with adversaries: Byzantine tolerant gradient
  descent.
\newblock In \emph{Proceedings of the 31st International Conference on Neural
  Information Processing Systems}, pages 118--128, 2017.

\bibitem[Yin et~al.(2018)Yin, Chen, Ramchandran, and
  Bartlett]{yin2018byzantine}
Dong Yin, Yudong Chen, Kannan Ramchandran, and Peter Bartlett.
\newblock Byzantine-robust distributed learning: Towards optimal statistical
  rates.
\newblock {arXiv}:1803.01498, 2018.

\bibitem[Alistarh et~al.(2018)Alistarh, Allen-Zhu, and
  Li]{alistarh2018byzantine}
Dan Alistarh, Zeyuan Allen-Zhu, and Jerry Li.
\newblock Byzantine stochastic gradient descent.
\newblock In \emph{Advances in Neural Information Processing Systems}, pages
  4613--4623, 2018.

\bibitem[Baruch et~al.(2019)Baruch, Baruch, and
  Goldberg]{baruch2019Byzantine-VA-attack}
Gilad Baruch, Moran Baruch, and Yoav Goldberg.
\newblock A little is enough: Circumventing defenses for distributed learning.
\newblock In H.~Wallach, H.~Larochelle, A.~Beygelzimer, F.~d\textquotesingle
  Alch\'{e}-Buc, E.~Fox, and R.~Garnett, editors, \emph{Advances in Neural
  Information Processing Systems}, volume~32. Curran Associates, Inc., 2019.
\newblock URL
  \url{https://proceedings.neurips.cc/paper/2019/file/ec1c59141046cd1866bbbcdfb6ae31d4-Paper.pdf}.

\bibitem[Khanduri et~al.(2019)Khanduri, Bulusu, Sharma, and
  Varshney]{khanduri2019byzantine}
Prashant Khanduri, Saikiran Bulusu, Pranay Sharma, and Pramod~K Varshney.
\newblock Byzantine resilient non-convex svrg with distributed batch gradient
  computations.
\newblock {arXiv}:1912.04531, 2019.

\bibitem[Allen-Zhu et~al.(2020)Allen-Zhu, Ebrahimian, Li, and
  Alistarh]{allen2020byzantine-iclr}
Zeyuan Allen-Zhu, Faeze Ebrahimian, Jerry Li, and Dan Alistarh.
\newblock Byzantine-resilient non-convex stochastic gradient descent.
\newblock {arXiv}:2012.14368, 2020.

\bibitem[Xie et~al.(2019)Xie, Koyejo, and Gupta]{xie2019zeno-byzantine}
Cong Xie, Sanmi Koyejo, and Indranil Gupta.
\newblock Zeno: Distributed stochastic gradient descent with suspicion-based
  fault-tolerance.
\newblock In \emph{International Conference on Machine Learning}, pages
  6893--6901. PMLR, 2019.

\bibitem[Metelli et~al.(2018)Metelli, Papini, Faccio, and
  Restelli]{2018importanceSampling}
Alberto~Maria Metelli, Matteo Papini, Francesco Faccio, and Marcello Restelli.
\newblock Policy optimization via importance sampling.
\newblock In \emph{Advances in Neural Information Processing Systems}, pages
  5442--5454, 2018.

\bibitem[Pirotta et~al.(2015)Pirotta, Restelli, and
  Bascetta]{pirotta2015Smooth}
Matteo Pirotta, Marcello Restelli, and Luca Bascetta.
\newblock Policy gradient in lipschitz markov decision processes.
\newblock \emph{Machine Learning}, 100\penalty0 (2-3):\penalty0 255--283, 2015.

\bibitem[Allen-Zhu and Hazan(2016{\natexlab{b}})]{allen2016variance}
Zeyuan Allen-Zhu and Elad Hazan.
\newblock Variance reduction for faster non-convex optimization.
\newblock In \emph{International conference on machine learning}, pages
  699--707. PMLR, 2016{\natexlab{b}}.

\bibitem[Albanie(2019)]{albanie2019euclidean}
Samuel Albanie.
\newblock Euclidean distance matrix trick.
\newblock Technical report, 2019.
\newblock URL
  \url{https://www.robots.ox.ac.uk/~albanie/notes/Euclidean_distance_trick.pdf}.

\bibitem[Barto et~al.(1983)Barto, Sutton, and Anderson]{barto1983CartPole}
Andrew~G Barto, Richard~S Sutton, and Charles~W Anderson.
\newblock Neuronlike adaptive elements that can solve difficult learning
  control problems.
\newblock \emph{IEEE transactions on systems, man, and cybernetics}, \penalty0
  (5):\penalty0 834--846, 1983.

\bibitem[Duan et~al.(2016)Duan, Chen, Houthooft, Schulman, and
  Abbeel]{duan2016TasksOnMujoco}
Yan Duan, Xi~Chen, Rein Houthooft, John Schulman, and Pieter Abbeel.
\newblock Benchmarking deep reinforcement learning for continuous control.
\newblock In \emph{International Conference on Machine Learning}, pages
  1329--1338, 2016.

\bibitem[Dubey and Pentland(2020)]{dubey2020differentially}
Abhimanyu Dubey and Alex Pentland.
\newblock Differentially-private federated linear bandits.
\newblock In \emph{Proc. {NeurIPS}}, 2020.

\bibitem[Shi et~al.(2021)Shi, Shen, and Yang]{shi2021federated}
Chengshuai Shi, Cong Shen, and Jing Yang.
\newblock Federated multi-armed bandits with personalization.
\newblock In \emph{Proc. {AISTATS}}, pages 2917--2925. PMLR, 2021.

\bibitem[Zhu et~al.(2021)Zhu, Zhu, Liu, and Liu]{zhu2021federated}
Zhaowei Zhu, Jingxuan Zhu, Ji~Liu, and Yang Liu.
\newblock Federated bandit: {A} gossiping approach.
\newblock In \emph{Abstract Proceedings of the 2021 ACM
  SIGMETRICS/International Conference on Measurement and Modeling of Computer
  Systems}, pages 3--4, 2021.

\bibitem[Dai et~al.(2020)Dai, Low, and Jaillet]{dai2020federated}
Zhongxiang Dai, Bryan Kian~Hsiang Low, and Patrick Jaillet.
\newblock Federated {Bayesian} optimization via {Thompson} sampling.
\newblock In \emph{Proc. {NeurIPS}}, 2020.

\bibitem[Dai et~al.(2021)Dai, Low, and Jaillet]{dai2021differentially}
Zhongxiang Dai, Bryan Kian~Hsiang Low, and Patrick Jaillet.
\newblock Differentially private federated {Bayesian} optimization with
  distributed exploration.
\newblock In \emph{Proc. {NeurIPS}}, 2021.

\bibitem[Sim et~al.(2021)Sim, Zhang, Low, and Jaillet]{sim2021collaborative}
Rachael Hwee~Ling Sim, Yehong Zhang, Bryan Kian~Hsiang Low, and Patrick
  Jaillet.
\newblock Collaborative {Bayesian} optimization with fair regret.
\newblock In \emph{Proc. {ICML}}, pages 9691--9701. PMLR, 2021.

\bibitem[Xu et~al.(2021{\natexlab{a}})Xu, Lyu, Ma, Miao, Foo, and
  Low]{xu2021gradient}
Xinyi Xu, Lingjuan Lyu, Xingjun Ma, Chenglin Miao, Chuan-Sheng Foo, and Bryan
  Kian~Hsiang Low.
\newblock Gradient driven rewards to guarantee fairness in collaborative
  machine learning.
\newblock In \emph{Proc. {NeurIPS}}, 2021{\natexlab{a}}.

\bibitem[Xu et~al.(2021{\natexlab{b}})Xu, Wu, Foo, and Low]{xu2021validation}
Xinyi Xu, Zhaoxuan Wu, Chuan-Sheng Foo, and Bryan Kian~Hsiang Low.
\newblock Validation free and replication robust volume-based data valuation.
\newblock In \emph{Proc. {NeurIPS}}, 2021{\natexlab{b}}.

\bibitem[Hoang et~al.(2021)Hoang, Hong, Xiao, Low, and Sun]{hoang2021aid}
Trong~Nghia Hoang, Shenda Hong, Cao Xiao, Bryan Kian~Hsiang Low, and Jimeng
  Sun.
\newblock Aid: {Active} distillation machine to leverage pre-trained black-box
  models in private data settings.
\newblock In \emph{Proc. {TheWebConf}}, pages 3569--3581, 2021.

\bibitem[Lam et~al.(2021)Lam, Hoang, Low, and Jaillet]{lam2021model}
Thanh~Chi Lam, Nghia Hoang, Bryan Kian~Hsiang Low, and Patrick Jaillet.
\newblock Model fusion for personalized learning.
\newblock In \emph{Proc. {ICML}}, pages 5948--5958. PMLR, 2021.

\bibitem[Sim et~al.(2020)Sim, Zhang, Chan, and Low]{sim2020collaborative}
Rachael Hwee~Ling Sim, Yehong Zhang, Mun~Choon Chan, and Bryan Kian~Hsiang Low.
\newblock Collaborative machine learning with incentive-aware model rewards.
\newblock In \emph{Proc. {ICML}}, pages 8927--8936. PMLR, 2020.

\bibitem[Hoang et~al.(2020)Hoang, Lam, Low, and Jaillet]{hoang2020learning}
Nghia Hoang, Thanh Lam, Bryan Kian~Hsiang Low, and Patrick Jaillet.
\newblock Learning task-agnostic embedding of multiple black-box experts for
  multi-task model fusion.
\newblock In \emph{Proc. {ICML}}, pages 4282--4292. PMLR, 2020.

\bibitem[Ouyang and Low(2020)]{ouyang2020gaussian}
Ruofei Ouyang and Bryan Kian~Hsiang Low.
\newblock Gaussian process decentralized data fusion meets transfer learning in
  large-scale distributed cooperative perception.
\newblock \emph{Autonomous Robots}, 44\penalty0 (3):\penalty0 359--376, 2020.

\bibitem[Hoang et~al.(2019{\natexlab{a}})Hoang, Hoang, Low, and
  Kingsford]{hoang2019collective}
Minh Hoang, Nghia Hoang, Bryan Kian~Hsiang Low, and Carleton Kingsford.
\newblock Collective model fusion for multiple black-box experts.
\newblock In \emph{Proc. {ICML}}, pages 2742--2750. PMLR, 2019{\natexlab{a}}.

\bibitem[Hoang et~al.(2019{\natexlab{b}})Hoang, Hoang, Low, and
  How]{hoang2019collective2}
Trong~Nghia Hoang, Quang~Minh Hoang, Bryan Kian~Hsiang Low, and Jonathan How.
\newblock Collective online learning of {Gaussian} processes in massive
  multi-agent systems.
\newblock In \emph{Proc. {AAAI}}, volume~33, pages 7850--7857,
  2019{\natexlab{b}}.

\bibitem[Ouyang and Low(2018)]{ouyang2018gaussian}
Ruofei Ouyang and Bryan Kian~Hsiang Low.
\newblock Gaussian process decentralized data fusion meets transfer learning in
  large-scale distributed cooperative perception.
\newblock In \emph{Proc. {AAAI}}, 2018.

\bibitem[Chen et~al.(2015)Chen, Low, Yao, and Jaillet]{chen2015gaussian}
Jie Chen, Bryan Kian~Hsiang Low, Yujian Yao, and Patrick Jaillet.
\newblock Gaussian process decentralized data fusion and active sensing for
  spatiotemporal traffic modeling and prediction in mobility-on-demand systems.
\newblock \emph{IEEE Transactions on Automation Science and Engineering},
  12\penalty0 (3):\penalty0 901--921, 2015.

\bibitem[Chen et~al.(2013)Chen, Low, and Tan]{chen2013gaussian}
Jie Chen, Bryan Kian~Hsiang Low, and Colin Keng-Yan Tan.
\newblock Gaussian process-based decentralized data fusion and active sensing
  for mobility-on-demand system.
\newblock In \emph{Proc. {RSS}}, 2013.

\bibitem[Chen et~al.(2012)Chen, Low, Tan, Oran, Jaillet, Dolan, and
  Sukhatme]{chen2012decentralized}
Jie Chen, Bryan Kian~Hsiang Low, Colin Keng-Yan Tan, Ali Oran, Patrick Jaillet,
  John~M Dolan, and Gaurav~S Sukhatme.
\newblock Decentralized data fusion and active sensing with mobile sensors for
  modeling and predicting spatiotemporal traffic phenomena.
\newblock In \emph{Proc. {UAI}}, 2012.

\bibitem[Pinelis(1994)]{pinelis1994optimumMartingale}
Iosif Pinelis.
\newblock Optimum bounds for the distributions of martingales in banach spaces.
\newblock \emph{The Annals of Probability}, pages 1679--1706, 1994.

\bibitem[Kingma and Ba(2014)]{kingma2014adam}
Diederik~P Kingma and Jimmy Ba.
\newblock Adam: A method for stochastic optimization.
\newblock {arXiv}:1412.6980, 2014.

\end{thebibliography}
\bibliographystyle{unsrtnat}

\newpage
\appendix


\section{More on the background}\label{appendix:more-background}
\subsection{SVRG and SCSG}\label{Appendix-SVRG-SCSG}
Here we provide the pseudocode for SVRG (Algorithm \ref{alg:svrg}) and SCSG (Algorithm \ref{alg:scsg}) seen in \citet{lei2017scsg2}. 
The idea of SVRG (Algorithm~\ref{alg:svrg}) is to reuses past \textit{full} gradient computations (line 3) to reduce the variance of the current \textit{stochastic} gradient estimate (line 7) before the parameter update (line 8). Note that $N=1$ corresponds to a GD step (i.e., $v_{k-1}^{(j)} \leftarrow g_j$ in line 7). For $N>1$, $v_{k-1}^{(j)}$ is the corrected gradient in SVRG and is an unbiased estimate of the true gradient $\nabla J(\boldsymbol{\theta})$. SVRG achieves linear convergence $O(1/T)$ using the semi-stochastic gradient.
\begin{algorithm}[H]
\renewcommand\thealgorithm{2}
   \caption{SVRG}
   \label{alg:svrg}
   \begin{algorithmic}[1]
   \STATE {\bfseries Input:} Number of stages $T$, initial iteratre $\tilde{\boldsymbol{\theta}}_0$, number of gradient steps $N$, step size $\eta$
   \FOR{$t=1$ {\bfseries to} $T$}
       \STATE $g_t \leftarrow \nabla J(\tilde{\boldsymbol{\theta}}_{t-1}) = \frac{1}{n}\sum_{i=1}^n \nabla J_i (\tilde{\boldsymbol{\theta}}_{t-1})$
       \STATE $\boldsymbol{\theta}_0^{(t)} \leftarrow \tilde{\boldsymbol{\theta}}_{t-1}$
       \FOR{$k=1$ {\bfseries to} $N$}
           \STATE Randomly pick $i_k \in [n]$
           \STATE {$v_{k-1}^{(t)} \leftarrow \nabla J_{i_{k}}(\boldsymbol{\theta}_{k-1}^{(t)})-\nabla J_{i_{k}}(\boldsymbol{\theta}_{0}^{(t)})+g_{t}$}
           \STATE ${\boldsymbol{\theta}}_{k}^{(t)} \leftarrow {\boldsymbol{\theta}}_{k-1}^{(t)} - \eta_t v_{k-1}^{(t)}$
       \ENDFOR
       \STATE $\tilde{{\boldsymbol{\theta}}}_{t} \leftarrow {\boldsymbol{\theta}}_{N_t}^{(t)}$    
   \ENDFOR
   \STATE {\bfseries Output:} $\tilde{\boldsymbol{\theta}}_T$ (Convex case) or $\tilde{\boldsymbol{\theta}}_t$ uniformly picked from $\{\tilde{\boldsymbol{\theta}}_t\}_{t=1}^T$ (Non-Convex case)
\end{algorithmic}
\end{algorithm}
More recently, \textit{Stochastically Controlled Stochastic Gradient} (SCSG) has been proposed~\citep{lei2016scsg1}, to further reduce the computational cost of SVRG. The key difference is that SCSG (Algorithm \ref{alg:scsg}) considers a sequence of time-varying batch sizes ($B_t$ and $b_t$) and employs geometric sampling to generate the number of parameter update steps $N_t$ in each iteration (line 6), instead of fixing the batch sizes and the number of updates as done in SVRG.
Particularly when finding an $\epsilon$-approximate solution (Definition \ref{definition:epsilon-approximate-solution}) for optimizing smooth non-convex objectives, \citet{lei2017scsg2} proves that SCSG is never worse than SVRG in convergence rate and significantly outperforms SVRG when the required $\epsilon$ is small. 
%
%
\begin{algorithm}[H]
\renewcommand\thealgorithm{3}
  \caption{SCSG for smooth non-convex objectives}
  \label{alg:scsg}
  \begin{algorithmic}[1]
  \STATE {\bfseries Input:} Number of stages $T$, initial iteratre $\tilde{\boldsymbol{\theta}}_0$, batch size $B_t$, mini-batch size $b_t$, step size $\eta_t$
  \FOR{$t=1$ {\bfseries to} $T$}
      \STATE Uniformly sample a batch $\mathcal{I}_t \subset\{1, \cdots, n\}$ with $|\mathcal{I}_{t}|=B_{t}$
      \STATE $g_t \leftarrow \nabla J_{\mathcal{I}_t}(\tilde{\boldsymbol{\theta}}_{t-1})$
      \STATE $\boldsymbol{\theta}_0^{(t)} \leftarrow \tilde{\boldsymbol{\theta}}_{t-1}$
      \STATE Generate $N_t \sim \operatorname{Geom}(B_t / (B_t + b_t))$
      \FOR{$k=1$ {\bfseries to} $N_t$}
          \STATE Randomly pick $\tilde{\mathcal{I}}_{k-1} \subset [n]$ with $|\tilde{\mathcal{I}}_{k-1}|=b_t$
          \STATE {$v_{k-1}^{(t)} \leftarrow \nabla J_{\tilde{\mathcal{I}}_{k-1}}(\boldsymbol{\theta}_{k-1}^{(t)})-\nabla J_{\tilde{\mathcal{I}}_{k-1}}(\boldsymbol{\theta}_{0}^{(t)})+g_{t}$}
          \STATE ${\boldsymbol{\theta}}_{k}^{(t)} \leftarrow {\boldsymbol{\theta}}_{k-1}^{(t)} + \eta_t v_{k-1}^{(t)}$
      \ENDFOR
      \STATE $\tilde{{\boldsymbol{\theta}}}_{t} \leftarrow {\boldsymbol{\theta}}_{N_t}^{(t)}$    
  \ENDFOR
  \STATE {\bfseries Output:} $\tilde{\boldsymbol{\theta}}_T$ (P-L case) or sample $\tilde{\boldsymbol{\theta}}^*_T$ from $\{\tilde{\boldsymbol{\theta}}_t\}_{t=1}^T$ with $P(\tilde{\boldsymbol{\theta}}_T^* = \tilde{\boldsymbol{\theta}}_t) \propto \eta_t B_t/b_t$ (Smooth case)
\end{algorithmic}
\end{algorithm}
As a member of the SVRG-like algorithms, SCSG enjoys the same convergence rate of SVRG while being computationally cheaper than SVRG for tasks with small $\epsilon$ requirements \citep{lei2016scsg1}, which is highly desired in RL, hence the motivation of FedPG-BR to adapt SCSG. 
\subsection{Gradient estimator}\label{appendix-gradient-estimator} 
Use $g(\tau|\boldsymbol{\theta})$ to denote the \textit{unbiased} estimator to the true gradient $J(\boldsymbol{\theta})$. The common gradient estimators are the REINFORCE and the GPOMDP estimators, which are considered as baseline estimators in \citep{papini2018stochastic} and \citep{xu2020improvedUAI}.
The REINFORCE \citep{williams1992REINFORCE}:
\begin{align*}
    g(\tau|\boldsymbol{\theta})=\left[\sum_{h=0}^{H-1} \nabla_{\boldsymbol{\theta}} \log \pi_{\boldsymbol{\theta}}(a_{h} \mid s_{h})\right]\left[\sum_{h=0}^{H-1} \gamma^{h} \mathcal{R}(s_{h}, a_{h})-C_b\right]
\end{align*}
And the GPOMDP \citep{baxter2001GPOMDP}
\begin{align*}
    g(\tau|\boldsymbol{\theta})=\sum_{h=0}^{H-1}\left[\sum_{t=0}^{h}\nabla_{\boldsymbol{\theta}} \log \pi_{\boldsymbol{\theta}}(a_{t} \mid s_{t})\right](\gamma^{h} r(s_{h}, a_{h})-C_{b_{h}})
\end{align*}
 where $C_b$ and $C_{b_h}$ are the corresponding baselines. Under Assumption \ref{assumption-policy-derivatives}, whether we use the REINFORCE or the GPOMDP estimator, Proposition \ref{proposition-uai-grad} holds \citep{papini2018stochastic, xu2020improvedUAI}.
\begin{algorithm}[H]
  \renewcommand\thealgorithm{4}
   \caption{GPOMDP (for federation of K agents)}
   \label{alg:Fed-GPOMDP}
\begin{algorithmic}
   \STATE {\bfseries Input:} number of iterations $T$, batch size $B$, step size $\eta$, initial parameter $\tilde{\boldsymbol{\theta}}_{0} \in \mathbb{R}^{d}$
   \FOR{$t=1$ {\bfseries to} $T$}
       \STATE $\boldsymbol{\theta}^{t} \leftarrow \tilde{\boldsymbol{\theta}}_{t-1}$ \qquad\qquad\qquad \qquad \qquad \qquad; broadcast to agents
       \FOR{$k=1$ {\bfseries to} $K$}
           \STATE Sample $B$ trajectories $\{\tau_{t,i}^{(k)}\}$ from $p(\cdot | \boldsymbol{\theta}^t)$
           \STATE $\mu_{t}^{(k)} = \frac{1}{B} \sum_{i=1}^B g(\tau_{t,i}^{(k)} | \boldsymbol{\theta}^{t})$ \qquad \qquad ; push $\mu_t^{(k)}$ to server
       \ENDFOR
       \STATE $\mu_t = \frac{1}{K}\sum_{k=1}^K\mu_t^{(k)}$ 
       \STATE $\tilde{\boldsymbol{\theta}}_{t} \leftarrow  \boldsymbol{\theta}^{t} + \eta \mu_{t}$
   \ENDFOR
   \STATE {\bfseries Output $\boldsymbol{\theta}_{out}$:} uniformly randomly picked from $\{\tilde{\boldsymbol{\theta}}_t\}_{t=1}^T$ 
\end{algorithmic}
\end{algorithm}
\subsection{Federated GPOMDP and SVRPG} \label{algo-fedVPG-fedSVRPG}
Closely following the problem setting of FedPG-BR, we adapt both GPOMDP and SVRPG to the FRL setting. The pseudocode is shown in Algorithm \ref{alg:Fed-GPOMDP} and Algorithm \ref{alg:fed-svrg}.

\begin{algorithm}[H]
\renewcommand\thealgorithm{5}
   \caption{SVRPG (for federation of K agents)}
   \label{alg:fed-svrg}
   \begin{algorithmic}
   \STATE {\bfseries Input:} number of epochs $T$, epoch size $N$, batch size $B$, mini-batch size $b$, step size $\eta$, initial parameter $\tilde{\boldsymbol{\theta}}_{0} \in \mathbb{R}^{d}$
   \FOR{$t=1$ {\bfseries to} $T$}
       \STATE $\boldsymbol{\theta}_{0}^{t} \leftarrow \tilde{\boldsymbol{\theta}}_{t-1}$ \qquad \qquad \qquad \qquad \qquad \qquad ; broadcast to agents
       \FOR{$k=1$ {\bfseries to} $K$}
           \STATE Sample $B$ trajectories $\{\tau_{t,i}^{(k)}\}$ from $p(\cdot | \boldsymbol{\theta}^t_0)$
           \STATE $\mu_{t}^{(k)} = \frac{1}{B} \sum_{i=1}^B g(\tau_{t,i}^{(k)} | \boldsymbol{\theta}^{t}_0)$ \qquad \qquad ; push $\mu_t^{(k)}$ to server
       \ENDFOR
       \STATE $\mu_t = \frac{1}{K}\sum_{k=1}^K\mu_t^{(k)}$
       \FOR{$n=0$ {\bfseries to} $N-1$}
           \STATE Sample $b$ trajectories $\{\tau_{n, j}^t\}$ from $p(\cdot|\boldsymbol{\theta}_n^t)$
           \STATE {$v_{n}^{t} = \frac{1}{b} \sum_{j=1}^{b} [g(\tau_{n,j}^{t}|\boldsymbol{\theta}_{n}^{t}) - \omega(\tau_{n,j}^t|\boldsymbol{\theta}_{n}^{t}, \boldsymbol{\theta}_{0}^{t}) g(\tau_{n,j}^{t}|\boldsymbol{\theta}_{0}^{t})] + \mu_t$}
           \STATE $\boldsymbol{\theta}_{n+1}^t = \boldsymbol{\theta}_{n}^{t} + \eta v_{n}^{t}$
       \ENDFOR
       \STATE $\tilde{\boldsymbol{\theta}}_{t} \leftarrow \boldsymbol{\theta}_{N}^{t}$    
   \ENDFOR
   \STATE {\bfseries Output $\boldsymbol{\theta}_{out}$:} uniformly randomly picked from $\{\tilde{\boldsymbol{\theta}}_t\}_{t=1}^T$ 
\end{algorithmic}
\end{algorithm}
\section{Proof of Theorem \ref{main-theorem}}\label{appendix:proof-of-theorem}
In our proof, we follow the suggestion from \citet{lei2017scsg2} to set $b_t = 1$ to derive better theoretical results. Refer to Section \ref{EXP-Settings-appendix} in this appendix for the value of $b_t$ used in our experiments.
\begin{proof}
From the L-smoothness of the objective function $J(\boldsymbol{\theta})$, we have
\begin{align*}
    \mathbb{E}_{\tau_n^t}[J(\boldsymbol{\theta}_{n+1}^t)] &\geq \mathbb{E}_{\tau_n^t}\left[J(\boldsymbol{\theta}_n^t) + \langle\nabla J(\boldsymbol{\theta}_n^t), \boldsymbol{\theta}_{n+1}^t - \boldsymbol{\theta}_n^t \rangle - \frac{L}{2}\|\boldsymbol{\theta}_{n+1}^t - \boldsymbol{\theta}_n^t\|^2\right] \\
    &= J(\boldsymbol{\theta}_n^t) + \eta_t \langle \mathbb{E}_{\tau_n^t}[v_n^t], \nabla J(\boldsymbol{\theta}_n^t)\rangle - \frac{L\eta_t^2}{2} \mathbb{E}_{\tau_n^t}[\|v_n^t\|^2] \\
                                           &\geq J(\boldsymbol{\theta}_n^t) + \eta_t \langle \nabla J(\boldsymbol{\theta}_n^t)+e_t, \nabla J(\boldsymbol{\theta}_n^t)\rangle \\
                                           & \qquad \qquad - \frac{L\eta_t^2}{2}[(2L_g + 2 C_g^2 C_w)\|\boldsymbol{\theta}_n^t - \boldsymbol{\theta}_0^t\|^2 + 2\|\nabla J(\boldsymbol{\theta}_n^t)\|^2 + 2\|e_t\|^2] \stepcounter{equation}\tag{\theequation}\label{main-smooth-1} \\
                                           &=J(\boldsymbol{\theta}_n^t) + \eta_t(1-L\eta_t)\|\nabla J(\boldsymbol{\theta}_n^t)\|^2 + \eta_t \langle e_t, \nabla J(\boldsymbol{\theta}_n^t) \rangle \\
                                           & \qquad \qquad - L\eta_t^2(L_g + C_g^2C_w)\|\boldsymbol{\theta}_n^t - \boldsymbol{\theta}_0^t\|^2 - L\eta_t^2\|e_t\|^2 
\end{align*}
where \eqref{main-smooth-1} follows from Lemma \ref{lemma-E-v}. Use $\mathbb{E}_t$ to denote the expectation with respect to all trajectories $\{\tau^t_1, \tau^t_2, ...\}$, given $N_t$. Since $\{\tau^t_1, \tau^t_2, ...\}$ are independent of $N_t$, $\mathbb{E}_t$ is equivalently the expectation with respect to $\{\tau^t_1, \tau^t_2, ...\}$. The above inequality gives
\begin{align*}
    \mathbb{E}_{t}[J(\boldsymbol{\theta}_{n+1}^t)] \geq \mathbb{E}_{t}[J(\boldsymbol{\theta}_n^t)] & + \eta_t(1-L\eta_t)\mathbb{E}_{t}\|\nabla J(\boldsymbol{\theta}_n^t)\|^2 + \eta_t \mathbb{E}_{t}\langle e_t, \nabla J(\boldsymbol{\theta}_n^t) \rangle \\ &- L\eta_t^2(L_g + C_g^2C_w)\mathbb{E}_{t}\|\boldsymbol{\theta}_n^t - \boldsymbol{\theta}_0^t\|^2 - L\eta_t^2\|e_t\|^2
\end{align*}
Taking $n=N_t$ and using $\mathbb{E}_{N_t}$ to denote the expectation w.r.t. $N_t$, we have from the above
\begin{align*}
    \mathbb{E}_{N_t}\mathbb{E}_{t}[J(\boldsymbol{\theta}_{N_t+1}^t)] \geq \mathbb{E}_{N_t}\mathbb{E}_{t}[J(\boldsymbol{\theta}_{N_t}^t)] &+ \eta_t(1-L\eta_t)\mathbb{E}_{N_t}\mathbb{E}_{t}\|\nabla J(\boldsymbol{\theta}_{N_t}^t)\|^2 + \eta_t \mathbb{E}_{N_t}\mathbb{E}_{t}\langle e_t, \nabla J(\boldsymbol{\theta}_{N_t}^t) \rangle \\ &- L\eta_t^2(L_g + C_g^2C_w)\mathbb{E}_{N_t}\mathbb{E}_{t}\|\boldsymbol{\theta}_{N_t}^t - \boldsymbol{\theta}_0^t\|^2 - L\eta_t^2\|e_t\|^2
\end{align*}
Rearrange,
\begin{align*}
    \eta_t(1-L\eta_t)\mathbb{E}_{N_t}\mathbb{E}_{t}\|\nabla J(\boldsymbol{\theta}_{N_t}^t)\|^2 &\leq \mathbb{E}_{N_t}\mathbb{E}_{t}[J(\boldsymbol{\theta}_{N_t+1}^t)] + L\eta_t^2(L_g + C_g^2C_w)\mathbb{E}_{N_t}\mathbb{E}_{t}\|\boldsymbol{\theta}_{N_t}^t - \boldsymbol{\theta}_0^t\|^2 \\
            & \quad -\eta_t \mathbb{E}_{N_t}\mathbb{E}_{t}\langle e_t, \nabla J(\boldsymbol{\theta}_{N_t}^t) \rangle + L\eta_t^2\|e_t\|^2 - \mathbb{E}_{N_t}\mathbb{E}_{t}[J(\boldsymbol{\theta}_{N_t}^t)] \\
            &= \frac{1}{B_{t}}(\mathbb{E}_{t} \mathbb{E}_{N_{t}}[ J(\boldsymbol{\theta}_{N_{t}}^{t})]-J(\boldsymbol{\theta}_{0}^{t}))-\eta_{t} \mathbb{E}_{N_{t}} \mathbb{E}_{t}\langle e_{t}, \nabla J(\boldsymbol{\theta}_{N_{t}}^{t})\rangle \\
            & \quad     + L\eta_{t}^{2}(L g+C_{g}^{2} C_{w}) \mathbb{E}_{N_{t}} \mathbb{E}_{t}\left\|\boldsymbol{\theta}_{N_{t}}^{t}-\boldsymbol{\theta}_{0}^{t}\right\|^{2}+L\eta_{t}^{2}\left\|e_{t}\right\|^{2} \stepcounter{equation}\tag{\theequation}\label{main-smooth-2} 
\end{align*}
where \eqref{main-smooth-2} follows from Lemma \ref{lemma-GeoSampling} with Fubini's theorem. Note that $\tilde{\boldsymbol{\theta}}_t = \boldsymbol{\theta}_{N_t}^t$ and $\tilde{\boldsymbol{\theta}}_{t-1} = \boldsymbol{\theta}_0^t$. If we take expectation over all the randomness and denote it by $\mathbb{E}$, we get
\begin{align*}
    \eta_{t}(1-L \eta_{t}) \mathbb{E}\|\nabla J(\tilde{\boldsymbol{\theta}}_{t})\|^{2}      &= \frac{1}{B_{t}} \mathbb{E}\left[J(\tilde{\boldsymbol{\theta}}_{t})-J(\tilde{\boldsymbol{\theta}}_{t-1})\right]-\eta_{t} \mathbb{E}\left\langle e_{t}, \nabla J(\tilde{\boldsymbol{\theta}}_{t})\right\rangle \\                                              & \quad \quad  +L\eta_t^2(L_g+C_g^2C_w) \mathbb{E}\|\tilde{\boldsymbol{\theta}}_{t}-\tilde{\boldsymbol{\theta}}_{t-1}\|^{2}+L \eta_{t}^{2} \mathbb{E}\|e_{t}\|^{2}  \\
                                                 &= \frac{1}{B_{t}} \mathbb{E}\left[J(\tilde{\boldsymbol{\theta}}_{t})-J(\tilde{\boldsymbol{\theta}}_{t-1})\right]-\frac{1}{B_{t}} \mathbb{E}\left\langle e_{t}, \tilde{\boldsymbol{\theta}}_{t}-\tilde{\boldsymbol{\theta}}_{t-1}\right\rangle \\
                                                 & \quad \quad  +L\eta_t^2(L_g+C_g^2C_w) \mathbb{E}\|\tilde{\boldsymbol{\theta}}_{t}-\tilde{\boldsymbol{\theta}}_{t-1}\|^{2}+\eta_{t}(1+L \eta_{t}) \mathbb{E}\left\|e_{t}\right\|^{2} \stepcounter{equation}\tag{\theequation}\label{main-lemma-eta-E}\\
                                                 & \leq \frac{1}{B_{t}} \mathbb{E}\left[J(\tilde{\boldsymbol{\theta}}_{t})-J(\tilde{\boldsymbol{\theta}}_{t-1})\right]\\
                                                 & \quad\quad  +\frac{1}{2 \eta_{t} B_{t}}[-\frac{1}{B_{t}}+\eta_t^2(2L_g+2C_g^2C_w)] \mathbb{E}\|\tilde{\boldsymbol{\theta}}_{t}-\tilde{\boldsymbol{\theta}}_{t-1}\|^{2} \\
                                                 & \quad \quad  + \frac{1}{B_{t}} \mathbb{E}\left\langle\nabla J(\tilde{\boldsymbol{\theta}}_{t}), \tilde{\boldsymbol{\theta}}_{t}-\tilde{\boldsymbol{\theta}}_{t-1}\right\rangle+\frac{\eta_{t}}{B_{t}} \mathbb{E}\|\nabla J(\tilde{\boldsymbol{\theta}}_{t})\|^{2}+\frac{\eta_{t}}{B_{t}} \mathbb{E}\left\|e_{t}\right\|^{2} \\
                                                 & \quad \quad  +L\eta_t^2(L_g+C_g^2C_w) \mathbb{E}\|\tilde{\boldsymbol{\theta}}_{t}-\tilde{\boldsymbol{\theta}}_{t-1}\|^{2}+\eta_{t}(1+L \eta_{t}) \mathbb{E}\left\|e_{t}\right\|^{2} \stepcounter{equation}\tag{\theequation}\label{main-lemma-neg-2eta-E}
\end{align*}
where \eqref{main-lemma-eta-E} follows from Lemma \ref{lemma-eta-E} and \eqref{main-lemma-neg-2eta-E} follows from Lemma \ref{lemma-neg-2eta-E}. Rearrange,
\begin{align*}
    \eta_{t}(1 & -  L \eta_{t} - \frac{1}{B_t}) \mathbb{E}\|\nabla J(\tilde{\boldsymbol{\theta}}_{t})\|^{2} + \frac{1 - 2\eta_t^2(L_g + C_g^2 C_w)B_t - 2L\eta_t^3(L_g + C_g^2C_w) B_t^2}{2\eta_t B_t^2} \mathbb{E}\|\tilde{\boldsymbol{\theta}}_{t}-\tilde{\boldsymbol{\theta}}_{t-1}\|^{2} \\ &\leq \frac{1}{B_t}\mathbb{E}\left[J(\tilde{\boldsymbol{\theta}}_{t})-J(\tilde{\boldsymbol{\theta}}_{t-1})\right] + \frac{1}{B_{t}} \mathbb{E}\left\langle\nabla J(\tilde{\boldsymbol{\theta}}_{t}), \tilde{\boldsymbol{\theta}}_{t}-\tilde{\boldsymbol{\theta}}_{t-1}\right\rangle +  \eta_t(1+L\eta_t + \frac{1}{B_t})\mathbb{E}\|e_t\|^2 \stepcounter{equation}\tag{\theequation}\label{lemma-proof:label-name-so-hard-1}
\end{align*}
Now we can apply Lemma \ref{lemma:Youngs-inequality} on $ \mathbb{E}\left\langle\nabla J(\tilde{\boldsymbol{\theta}}_{t}), \tilde{\boldsymbol{\theta}}_{t}-\tilde{\boldsymbol{\theta}}_{t-1}\right\rangle$ using $a=\tilde{\boldsymbol{\theta}}_t - \tilde{\boldsymbol{\theta}}_{t-1}$, $b=\nabla J(\tilde{\boldsymbol{\theta}}_t)$, and $\beta = \frac{1-2\eta_t^2(L_g + C_g^2C_w)B_t - 2L\eta_t^3(L_g + C_g^2C_w) B_t^2}{\eta_t B_t}$ to get
\begin{align*}
    \frac{1}{B_{t}} \mathbb{E}\left\langle\nabla J(\tilde{\boldsymbol{\theta}}_{t}), \tilde{\boldsymbol{\theta}}_{t}-\tilde{\boldsymbol{\theta}}_{t-1}\right\rangle &\leq \frac{1-2\eta_t^2(L_g + C_g^2C_w)B_t - 2L\eta_t^3(L_g + C_g^2C_w) B_t^2}{2\eta_t B_t^2}\mathbb{E}\|\tilde{\boldsymbol{\theta}}_{t}-\tilde{\boldsymbol{\theta}}_{t-1}\|^2 \\
    &+ \frac{\eta_t}{2[1-2\eta_t^2(L_g + C_g^2C_w)B_t - 2L\eta_t^3(L_g + C_g^2C_w) B_t^2]}\mathbb{E}\|\nabla J(\tilde{\boldsymbol{\theta}}_{t})\|^2 \stepcounter{equation}\tag{\theequation}\label{lemma-proof:label-name-so-hard-2}
\end{align*}
Combining \eqref{lemma-proof:label-name-so-hard-1} and \eqref{lemma-proof:label-name-so-hard-2} and rearrange, we have
\begin{align*}
    \eta_{t}(1- & L \eta_{t}  - \frac{1}{B_t} -  \frac{1}{2[1-2\eta_t^2(L_g + C_g^2C_w)B_t - 2L\eta_t^3(L_g + C_g^2C_w) B_t^2]}) \mathbb{E}\|\nabla J(\tilde{\boldsymbol{\theta}}_{t})\|^{2} \\ &\leq \frac{1}{B_t}\mathbb{E}\left[J(\tilde{\boldsymbol{\theta}}_{t})-J(\tilde{\boldsymbol{\theta}}_{t-1})\right] + \eta_t(1+L\eta_t + \frac{1}{B_t})\mathbb{E}\|e_t\|^2 \\ &\leq \frac{1}{B_t}\mathbb{E}\left[J(\tilde{\boldsymbol{\theta}}_{t})-J(\tilde{\boldsymbol{\theta}}_{t-1})\right] + \eta_t(1+L\eta_t + \frac{1}{B_t})\left[\frac{4\sigma^2}{(1-\alpha)^2KB_t} + \frac{48\alpha^2\sigma^2V}{(1-\alpha)^2B_t}\right] \stepcounter{equation}\tag{\theequation}\label{main-final-1}
\end{align*}
where \eqref{main-final-1} follows from Lemma \ref{lemma-error-bound}.
We want to choose $\eta_t$ such that $1-2\eta_t^2(L_g + C_g^2C_w)B_t - 2L\eta_t^3(L_g + C_g^2C_w) B_t^2 > 0$. Denoting $\Phi = L_g + C_g^2C_w$, we have
\begin{align*}
    \textcolor{red}{1>1-2\eta_t^2\Phi B_t - 2L\eta_t^3\Phi B_t^2 }&\textcolor{red}{> 0} \\
    \textcolor{red}{\frac{1}{2(1-2\eta_t^2\Phi B_t - 2L\eta_t^3\Phi B_t^2)}} &\textcolor{red}{> \frac{1}{2}}
\end{align*}
We can then choose $\eta_t$ to have
\begin{align*}
    1- L \eta_{t}  - \frac{1}{B_t} -  \frac{1}{2[1-2\eta_t^2\Phi B_t - 2L\eta_t^3\Phi B_t^2]}  & \geq \frac{1}{4} \stepcounter{equation}\tag{\theequation}\label{eq:eta_1}
    \\
    L \eta_{t}  + \frac{1}{B_t} +  \frac{1}{2[1-2\eta_t^2\Phi B_t - 2L\eta_t^3\Phi B_t^2]} &\leq \frac{3}{4}
\end{align*}
We thus want 
\begin{align*}
    \text{(i)} \frac{1}{2} <  \frac{1}{2[1-2\eta_t^2\Phi B_t - 2L\eta_t^3\Phi B_t^2]}  & \leq \frac{5}{8} \\ 
    \text{(ii)} L\eta_t \leq \frac{1}{16} \qquad\qquad\qquad\quad \text{(iii)} \frac{1}{B_t} & \leq \frac{1}{16}
\end{align*}
where (i) implies 
\begin{align*}
    \eta_t^2\Phi B_t + L\eta_t^3 \Phi B_t^2 &\leq \frac{1}{10}\\
    \rightarrow \eta_t^2\Phi B_t \leq \frac{1}{20} \qquad &\& \qquad L\eta_t^3\Phi B_t^2 \leq \frac{1}{20} \\
    \Rightarrow \eta_t \leq \frac{1}{20^{1/2}\Phi^{1/2}B_t^{1/2}} \quad &\& \quad \eta_t \leq \frac{1}{20^{1/3}\Phi^{1/3}B_t^{2/3}L^{1/3}}
\end{align*}
From (ii) \& (iii)
\begin{align*}
    \eta_t \leq \frac{1}{16L} \qquad \& \qquad B_t \geq 16 \\
\end{align*}
We can then choose $\eta_t \leq \frac{1}{2\Psi B_t^{2/3}}$, where $\Psi = (L\Phi)^{1/3} = (L(L_g + C_g^2C_w))^{1/3}$, s.t.
\begin{align*}
    \frac{1}{2(L\Phi)^{1/3}B_t^{2/3}} \leq \min\{\frac{1}{16}, 
    \frac{1}{20^{1/2}\Phi^{1/2}B_t^{1/2}}, \frac{1}{20^{1/3}(\Phi L)^{1/3}B_t^{2/3}}\}
\end{align*}
One example of such choice is setting $\eta_t = \frac{1}{2\Psi B_t^{2/3}}$ when $B_t \geq 16$ which satisfies \eqref{eq:eta_1}.
We can obtain the following from \eqref{main-final-1} and \eqref{eq:eta_1}:
\begin{align*}
                                                      %
    \textcolor{red}{\frac{1}{4}}\eta_{t}\mathbb{E}\|\nabla J(\tilde{\boldsymbol{\theta}}_{t})\|^{2} \leq \frac{1}{B_t}\mathbb{E}\left[J(\tilde{\boldsymbol{\theta}}_{t})-J(\tilde{\boldsymbol{\theta}}_{t-1})\right] + 2\eta_t \left[\frac{4\sigma^2}{(1-\alpha)^2KB_t} + \frac{48\alpha^2\sigma^2V}{(1-\alpha)^2B_t}\right]
\end{align*}

Replacing $\eta_t = \frac{1}{2\Psi B_t^{2/3}}$ and rearranging, we have
\begin{align*}
    \mathbb{E}\|\nabla J(\tilde{\boldsymbol{\theta}}_{t})\|^{2} &\leq 
    \textcolor{red}{4}\Big[ \frac{1}{B_{t}\eta_t} \mathbb{E}\left[J(\tilde{\boldsymbol{\theta}}_{t})-J(\tilde{\boldsymbol{\theta}}_{t-1})\right]+2 \left[\frac{4 \sigma^{2}}{(1-\alpha)^{2} K B_{t}}+\frac{48 \alpha^{2} \sigma^{2} V}{(1-\alpha)^{2} B_{t}}\right] \Big]\\
                                                              &
    \leq \textcolor{red}{4}\Big[
    \frac{2 \Psi \mathbb{E}\left[J(\tilde{\boldsymbol{\theta}}_{t})-J(\tilde{\boldsymbol{\theta}}_{t-1})\right]}{B_{t}^{1 / 3}}+\frac{8 \sigma^{2}}{(1-\alpha)^{2} K B_{t}}+\frac{96 \alpha^{2} \sigma^{2} V}{(1-\alpha)^{2} B_{t}} \Big]
\end{align*}

Replacing $B_t$ with constant batch size $B$ and telescoping over $t = 1,2,...,T$, we have for $\tilde{\boldsymbol{\theta}}_{a}$ from our algorithm:
\begin{align*}
    \mathbb{E}\|\nabla J(\tilde{\boldsymbol{\theta}}_{a})\|^{2} &\leq \textcolor{red}{4}\Big[
    {\frac{2 \Psi \mathbb{E}\left[J(\tilde{\boldsymbol{\theta}}_T)-J(\tilde{\boldsymbol{\theta}}_{0})\right]}{TB^{1 / 3}}}+ {\frac{8 \sigma^{2}}{(1-\alpha)^{2} K B}}+ {\frac{96 \alpha^{2} \sigma^{2} V}{(1-\alpha)^{2} B}} \Big]\\
    &\leq \textcolor{red}{4}\Big[
    {\frac{2 \Psi \left[J(\tilde{\boldsymbol{\theta}}^*)-J(\tilde{\boldsymbol{\theta}}_{0})\right]}{TB^{1 / 3}}}+ {\frac{8 \sigma^{2}}{(1-\alpha)^{2} K B}}+ {\frac{96 \alpha^{2} \sigma^{2} V}{(1-\alpha)^{2} B}} \Big]
\end{align*}
which completes the proof.

\textcolor{red}{Note:} The part highlighted in red is different from the corresponding part in our NeurIPS proceedings and is clearer when reasoning with the choice of $\eta_t$. It added another constant factor to the RHS of Theorem~\ref{main-theorem} but does not affect our Corollary~\ref{main-corollary} which is the main result of our story.
\end{proof}

\section{Proof of Corollary \ref{main-corollary}}\label{Appendix:proof-of-corollary}
\begin{proof}
Recall $\Psi = (L(L_g + C_g^2C_w))^{1/3}$. From Theorem \ref{main-theorem}, we have
\begin{align*}
    \mathbb{E}\|\nabla J(\tilde{\boldsymbol{\theta}}_{a})\|^{2} \leq \underbrace{\frac{2 \Psi \mathbb{E}\left[J(\tilde{\boldsymbol{\theta}}^*)-J(\tilde{\boldsymbol{\theta}}_{0})\right]}{TB^{1 / 3}}}_{T=O(\frac{1}{\epsilon B^{1 / 3}})}+ \underbrace{\frac{8 \sigma^{2}}{(1-\alpha)^{2} K B}}_{B_K=O(\frac{1}{\epsilon K})}+ \underbrace{\frac{96 \alpha^{2} \sigma^{2} V}{(1-\alpha)^{2} B}}_{B_{\alpha}=O(\frac{\alpha^{2}}{\epsilon})}
\end{align*}
To guarantee that the output of Algorithm \ref{alg:FedPG-BR} is $\epsilon$-approximate, i.e., $\mathbb{E}\|\nabla J(\boldsymbol{\tilde{\theta}}_a)\|^2 \leq \epsilon$, we need the number of rounds $T$ and the batch size $B$ to meet the following:
\begin{align*}
    \text{(i)} T = O(\frac{1}{\epsilon B^{1/3}}), \text{(ii)} B_K=O(\frac{1}{\epsilon K}) \text{, and (iii)} B_{\alpha}=O(\frac{\alpha^2}{\epsilon}) 
\end{align*}
By union bound and using $\mathbb{E}[Traj(\epsilon)]$ to denote the total number of trajectories required by each agent to sample, the above implies that
\begin{align*}
    \mathbb{E}[Traj(\epsilon)] &\leq TB_K + TB_{\alpha} \\
     &\leq O(\frac{1}{\epsilon^{5/3}K^{2/3}} + \frac{\alpha^{4/3}}{\epsilon^{5/3}})
\end{align*}
in order to obtain an $\epsilon$-approximate policy, which completes the proof for Corollary \ref{main-corollary} \textit{(i)}. Note that the total number of trajectories generated across the whole FRL system, denoted by $\mathbb{E}[Traj_{total}(\epsilon)]$ is thus bounded by:
\begin{align*}
    \mathbb{E}[Traj_{total}(\epsilon)] &\leq O(\frac{K^{1/3}}{\epsilon^{5/3}} + \frac{K\alpha^{4/3}}{\epsilon^{5/3}})
\end{align*}
Now for an ideal system where $\alpha = 0$:
\begin{align*}
    \mathbb{E}[Traj(\epsilon)] &\leq O(\frac{1}{\epsilon^{5/3}K^{2/3}}) \\
    \mathbb{E}[Traj_{total}(\epsilon)] &\leq O(\frac{K^{1/3}}{\epsilon^{5/3}})
\end{align*}
which completes the proof for Corollary \ref{main-corollary} \textit{(ii)}. Moreover, when $K = 1$, the number of trajectories required by the agent using FedPG-BR is 
\begin{align*}
    \mathbb{E}[Traj(\epsilon)] &\leq O(\frac{1}{\epsilon^{5/3}}) 
\end{align*}
which is Corollary \ref{main-corollary} \textit{(iii)} and is coherent with the recent analysis of SVRPG \citep{xu2020improvedUAI}.
\end{proof}
\section{More on the Byzantine Filtering Step}
\label{proof-claims-in-filtering-strategy}
In this section, we continue our discussion on our \textit{Byzantine Filtering Step} in Section \ref{subsection:algorithm-description}. We include the pseudocode for the subroutine \textbf{FedPG-Aggregate} below for ease of reference:

\begin{algorithm}[ht]
   \renewcommand\thealgorithm{1.1}
   \caption{\textbf{FedPG-Aggregate}}
\begin{algorithmic}[1]
  \STATE {\bfseries Input:} Gradient estimates from $K$ agents in round $t$: $\{\mu_t^{(k)}\}_{k=1}^k$, Variance Bound $\sigma$, filtering threshold $\mathfrak{T}_{\mu}\triangleq2 \sigma \sqrt{\frac{V}{B_t}}$, where $V\triangleq2\operatorname{log}(\frac{2K}{\delta})$ and $\delta \in (0,1)$
        \STATE {$S_1 \triangleq  \{\mu_{t}^{(k)}\} \text { where } k \in[K] \text { s.t. } 
         \left|\left\{k^{\prime}\in[K]:\left\|\mu_{t}^{(k^{\prime})}-\mu_{t}^{(k)}\right\| \leq \mathfrak{T}_{\mu}\right\}\right|>\frac{K}{2}$}
        \STATE {$\mu_{t}^{\text {mom }} \leftarrow \operatorname{argmin}_{\mu_t^{(\tilde{k})}}\|\mu_t^{(\tilde{k})} - \operatorname{mean}(S_1)\| \text{ where } \tilde{k} \in S_1 $}
        \STATE \textit{R1:} $\mathcal{G}_{t} \triangleq \left\{k \in[K]:\left\|\mu_{t}^{(k)}-\mu_{t}^{\text{mom}}\right\| \leq \mathfrak{T}_{\mu}\right\}$
      
        \IF{$\left|\mathcal{G}_{t}\right|<(1-\alpha) K$ \quad} 
            \STATE {$S_2 \triangleq  \{\mu_{t}^{(k)}\} \text { where } k \in[K] \text { s.t. } 
             \left|\left\{k^{\prime}\in[K]:\left\|\mu_{t}^{(k^{\prime})}-\mu_{t}^{(k)}\right\| \leq 2 \sigma\right\}\right|>\frac{K}{2}$}
            \STATE {$\mu_{t}^{\text {mom }} \leftarrow \operatorname{argmin}_{\mu_t^{(\tilde{k})}}\|\mu_t^{(\tilde{k})} - \operatorname{mean}(S_2)\| \text{ where } \tilde{k} \in S_2 $}
            \STATE \textit{R2:} $\mathcal{G}_{t} \triangleq \left\{k \in[K]:\left\|\mu_{t}^{(k)}-\mu_{t}^{\text{mom}}\right\| \leq 2 \sigma\right\}$
        \ENDIF
      \STATE {\bfseries Return:} $\mu_{t} \triangleq \frac{1}{\left|\mathcal{G}_{t}\right|} \sum_{k \in \mathcal{G}_{t}} \mu_{t}^{(k)}$ 
\end{algorithmic}
\end{algorithm}

As discussed in Section \ref{subsection:algorithm-description}, R2 (line 8 in Algorithm \ref{alg:FedAgg-BR}) ensures that $\mathcal{G}_t$ always include all good agents and for any Byzantine agents being included, their impact on the convergence of Algorithm \ref{alg:FedPG-BR} is limited since their maximum distance to $\nabla J(\boldsymbol{\theta}_0^t)$ is bounded by $3\sigma$. Here we give proofs for the claims.
%
%
%
%
%
%
\begin{claim} \label{claim:R2-ensures-all-good-agents-are-included}
Under Assumption \ref{assumption-bounded-variance} and $\forall \alpha < 0.5$, the filtering rule R2 in Algorithm \ref{alg:FedAgg-BR} ensures that, in any round $t$, all gradient estimates sent from non-Byzantine agents are included in $\mathcal{G}_t$, i.e., $|\mathcal{G}_t| \geq (1-\alpha)K$.
\end{claim}
\begin{proof}
First, from Assumption \ref{assumption-bounded-variance}:
\begin{align*}
    \|\mu_t^{(k)} - \nabla J(\boldsymbol{\theta}_0^t)\| \leq \sigma, \forall k \in \mathcal{G}
\end{align*}
it implies that $\|\mu_t^{(k_1)} - \mu_t^{(k_2)}\| \leq 2\sigma, \forall k_1, k_2 \in \mathcal{G}$. So, for any value of the \textit{vector median} \citep{alistarh2018byzantine} in S2 $=\{\mu_t^{(k)}\}$ (defined in line 6):
\begin{align*}
    \|\mu_t^{(k)} - \nabla J(\boldsymbol{\theta}_0^t)\| \leq 3\sigma, \forall \mu_t^{(k)} \in \text{S2}
\end{align*}
An intuitive illustration is provided in Fig.~\ref{fig:filtering-diagram}. Next, consider the worst case where all values sent by the $K$ agents are included in S2: for all $(1-\alpha)K$ good agents, they send the same value $\mu_t^{(k)}$, s.t., $\|\nabla J(\boldsymbol{\theta}_0^t) - \mu_t^{(k)}\| = \sigma, \forall k \in \mathcal{G}$; and for all $\alpha K$ Byzantine agents, they send the same value $\mu_t^{(k^{\prime})}$ s.t., $\|\nabla J(\boldsymbol{\theta}_0^t) - \mu_t^{(k^{\prime})}\| = 3\sigma, \forall k^{\prime} \in S2\setminus \mathcal{G}$. Then the mean of values in S2 satisfies:
\begin{align*}
    \|\mu_t^{\operatorname{mean}} - \nabla J({\boldsymbol{\theta}_0^t})\| &= \frac{(1-\alpha)K\cdot\sigma + \alpha K \cdot 3\sigma}{K} \\
                                                                     &= (1-\alpha)\sigma + 3\alpha\sigma \\
                                                                     &= \sigma + 2\alpha\sigma \\
                                                                     &< 2\sigma
\end{align*}        
where the last inequality holds for $\alpha < 0.5$ which is our assumption. Then the value $\mu_t^{\operatorname{mom}}$ of Algorithm \ref{alg:FedPG-BR} will be set to any $\mu_t^{(k)}$ from S2, of which is the closet to $\mu_t^{\operatorname{mean}}$.

The selection of $\mu_t^{\operatorname{mom}} $ implies $\|\mu_t^{\operatorname{mom}} - \nabla J(\boldsymbol{\theta}_0^t)\| \leq \sigma$. Therefore, by constructing a region of $\mathcal{G}_t$ that is centred at $\mu_t^{\operatorname{mom}}$ and $2\sigma$ in radius (line 8), $\mathcal{G}_t$ can cover all estimates from non-Byzantine agents and hence ensure $|\mathcal{G}_t| \geq (1-\alpha)K$.
\end{proof}
\begin{claim} \label{claim2:distance-in-filtering}
Under Assumption \ref{assumption-bounded-variance} and $\alpha < 0.5$, the filtering rule R2 in Algorithm \ref{alg:FedAgg-BR} ensures that, in any round $t$, $\|\mu_t^{(k)} - \nabla J(\boldsymbol{\theta}_0^t)\| \leq 3\sigma, \forall k \in \mathcal{G}_t$.
\end{claim}
\begin{proof}
This lemma is a straightforward result following the proof of Claim \ref{claim:R2-ensures-all-good-agents-are-included}.
\end{proof}
\begin{remark}
Claim \ref{claim2:distance-in-filtering} implies that, in any round $t$, if an estimate sent from Byzantine agent is included in $\mathcal{G}_t$, then its impact on the convergence of Algorithm \ref{alg:FedPG-BR} is limited since its distance to $\nabla J(\boldsymbol{\theta}_0^t)$ is bounded by $3\sigma$. Fig.~\ref{fig:filtering-diagram} provides an intuitive illustration for this claim.
\end{remark}

\begin{figure}[H]
    \centering
    \includegraphics[scale=0.5]{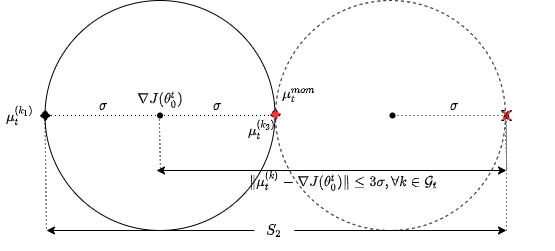}
    \caption{Graphical illustration of the Byzantine filtering strategy where $\mu_t^{(k_1)}, \mu_t^{(k_2)}$ are two good gradients while the red cross represents one Byzantine gradient which falls within $S_2$. $\mu_t^{\text{mom}}$ will be chosen at the red diamond.}
    \label{fig:filtering-diagram}
\end{figure}
As discussed above, R2 ensures that all good agents are included in $\mathcal{G}_t$, i.e., a region in which all good agents are \textit{concentrated}. R1 (lines 2-4) is designed in a similar way and aims to improve the practical performance of FedPG-BR by exploiting Lemma \ref{lemma-martingale}: all good agents are \textit{highly likely} to be \textit{concentrated} in a much \textit{smaller region}. 
\begin{claim}
Define $V \triangleq 2\operatorname{log}(2K/\delta)$ and $\delta \in (0,1)$, the filtering R1 in Algorithm \ref{alg:FedPG-BR} ensure
\begin{align*}
    \|\mu_t^{(k)} - \nabla J(\boldsymbol{\theta}_0^t)\| \leq \sigma \sqrt{\frac{V}{B_t}}, \forall k \in \mathcal{G}
\end{align*}
with probability of at least $1 - \delta$.
\end{claim}
\begin{proof}
From Assumption \ref{assumption-bounded-variance}, $\|\mu_t^{(k)} - \nabla J(\boldsymbol{\theta}_0^t)\| \leq \sigma, \forall k \in \mathcal{G}$. We have
\begin{align*}
    \|\mu_t^{(k)} - \nabla J(\boldsymbol{\theta}_0^t)\| &= \left\|\frac{1}{B_t}\sum_{i=1}^{B_t} g(\tau_{t,i}^{(k)} | \boldsymbol{\theta}_0^t) - \nabla J(\boldsymbol{\theta}_0^t)\right\| \\
                                                    &= \frac{1}{B_t}\sqrt{\left\|\sum_{i=1}^{B_t} g(\tau_{t,i}^{(k)} | \boldsymbol{\theta}_0^t) - \nabla J(\boldsymbol{\theta}_0^t)\right\|^2} \stepcounter{equation}\tag{\theequation}\label{proof:claim-in-martingale}
\end{align*}
Consider $X_i \triangleq g(\tau_{t,i}^{(k)}) - \nabla J(\boldsymbol{\theta}_0^t)$ and apply Lemma \ref{lemma-martingale} on \eqref{proof:claim-in-martingale}, we have
\begin{align*}
    \Pr\left[\left\|\sum_{i=1}^{B_t}X_i\right\|^2 \leq 2\operatorname{log}(\frac{2}{\delta})\sigma^2 B_t\right] \geq 1 - \delta \\
    \Pr\left[\frac{1}{B_t}\sqrt{\left\|\sum_{i=1}^{B_t}X_i\right\|^2} \leq \frac{1}{B_t} \sqrt{2\operatorname{log}(\frac{2}{\delta})\sigma^2 B_t}\right] \geq 1 - \delta
\end{align*}
With $V \triangleq 2\operatorname{log}(2K/\delta)$ and $\delta \in (0,1)$, the above inequality yields the Claim.
\end{proof}
Therefore, the first filtering R1 (lines 2-4) of FedPG-BR constructs a region of $\mathcal{G}_t$ centred at $\mu_t^{\operatorname{mom}}$ with radius of $2\sigma\sqrt{\frac{V}{B_t}}$, which ensures in any round $t$ that, \textit{with probability} $\geq 1-\delta$, (a) all good agents are included in $\mathcal{G}_t$, and (b) if gradients from Byzantine agents are included in $\mathcal{G}_t$, their impact is limited since their maximum distance to $\nabla J(\boldsymbol{\theta}_0^t)$ is bounded by $3\sigma \sqrt{\frac{V}{B_t}}$ (The proof is similar to that of Claim \ref{claim2:distance-in-filtering}). Compared to R2, R1 can construct a \textit{smaller} region that the server believes contains all good agents. If any Byzantine agent is included, their impact is also \textit{smaller}, with probability of at least $1-\delta$. Therefore, R1 is applied first such that if R1 fails (line 5) which happens with probability $< \delta$, R2 is then employed as a backup to ensure that $\mathcal{G}_t$ always includes all good agents. 
%
%

\section{Useful technical lemmas}
\label{app:some_useful_lemmas}
\begin{lemma}[Unbiaseness of importance sampling]
\label{lemma-unbiased-IS}
\begin{align*}
    \mathbb{E}_{\tau \sim p\left(\cdot \mid \boldsymbol{\theta}_{n}\right)} [ \omega(\tau|\boldsymbol{\theta}_n,\boldsymbol{\theta}_0)g(\tau|\boldsymbol{\theta}_0)] &= \mathbb{E}_{\tau \sim p(\cdot | \boldsymbol{\theta}_{0})}[g(\tau|\boldsymbol{\theta}_0)] \\
                                                    &=\nabla J({\boldsymbol{\theta}_0})
\end{align*}
\end{lemma}
\begin{proof}
Drop $t$ from notation and use $\tau_n$ to denote trajectories sampled from $\boldsymbol{\theta}_n$ at step $n$. From the definition of gradient estimation, we have
    \begin{align*}
        g(\tau_n | \boldsymbol{\theta}_0) &= \mathbb{E}_{\tau \sim p(\cdot | \boldsymbol{\theta}_n)}[\nabla_{\boldsymbol{\theta}_0}p(\boldsymbol{\theta}_0)r(\tau)] \\
                            &= \int p(\cdot | \boldsymbol{\theta}_n) \nabla_{\boldsymbol{\theta}_0} p(\boldsymbol{\theta}_0) r(\tau) d\tau \\
                            &= \int \frac{p(\cdot | \boldsymbol{\theta}_0)}{p(\cdot | \boldsymbol{\theta}_0)} p(\cdot | \boldsymbol{\theta}_n) \nabla_{\boldsymbol{\theta}_0} p(\boldsymbol{\theta}_0) r(\tau) d\tau \\
                            &= \int p(\cdot | \boldsymbol{\theta}_0) \frac{p(\cdot | \boldsymbol{\theta}_n)}{p(\cdot | \boldsymbol{\theta}_0)} \nabla_{\boldsymbol{\theta}_0} p(\boldsymbol{\theta}_0) r(\tau) d\tau \\
                            &= \mathbb{E}_{\tau \sim p(\cdot | \boldsymbol{\theta}_0)} \left[\frac{p(\cdot | \boldsymbol{\theta}_n)}{p(\cdot | \boldsymbol{\theta}_0)} \nabla_{\boldsymbol{\theta}_0} p(\boldsymbol{\theta}_0) r(\tau)\right] \\
                            &= \frac{p(\cdot | \boldsymbol{\theta}_n)}{p(\cdot | \boldsymbol{\theta}_0)} g(\tau_0 | \boldsymbol{\theta}_0)
    \end{align*}
Then,
    \begin{align*}
        \omega(\tau | \boldsymbol{\theta}_n, \boldsymbol{\theta}_0)g(\tau_n | \boldsymbol{\theta}_0) &= \frac{p(\cdot | \boldsymbol{\theta}_0)}{p(\cdot | \boldsymbol{\theta}_n)} g(\tau_n | \boldsymbol{\theta}_0) \\
        &= g(\tau_0 | \boldsymbol{\theta}_0)
    \end{align*}
which gives the lemma.
\end{proof}

\begin{lemma}[Adapted from \citep{xu2020improvedUAI}]
    \label{lemma-uai-variance-IS}
    Let $\omega(\tau | \boldsymbol{\theta}_1, \boldsymbol{\theta}_2) = p(\tau | \boldsymbol{\theta}_1) / p(\tau | \boldsymbol{\theta}_2)$, under Assumptions \ref{assumption-policy-derivatives} and \ref{assumption-uai-weight}, it holds that
    \begin{align*}
     Var(\omega(\tau|\boldsymbol{\theta}_1, \boldsymbol{\theta}_2)) \leq C_w\|\boldsymbol{\theta}_1 - \boldsymbol{\theta}_2\|^2  
    \end{align*}
    where $C_w = H(2HG^2 + M)(W+1) $. Furthermore, we have
    \begin{align*}
        \mathbb{E}_{\tau_n^t} & \|1 - \omega(\tau_n^t|\boldsymbol{\theta}_n^t, \boldsymbol{\theta}_0^t)\|^2 \\
        &= Var_{\boldsymbol{\theta}_n^t, \boldsymbol{\theta}_0^t}(\omega(\tau_n^t|\boldsymbol{\theta}_n^t, \boldsymbol{\theta}_0^t)) \\
        &\leq C_w\|\boldsymbol{\theta}_n^t - \boldsymbol{\theta}^t_0\|^2
    \end{align*}
\end{lemma}
\begin{proof}
The proof can be found in \citet{xu2020improvedUAI}.
\end{proof}
\begin{lemma}
\label{lemma-usefulLemma}
For $X_1, X_2 \in \mathbb{R}^d$, we have
\begin{align*}
    \|X_1 + X_2\|^2 \leq 2\|X_1\|^2 + 2\|X_2\|^2
\end{align*}
\end{lemma}

\begin{lemma}
\label{lemma-E-v}
\begin{align*}
    \mathbb{E}_{\tau_n^t}[\|v_n^t\|^2] \leq (2 L_{g}+2C_g^{2} C_{w})\left\|\boldsymbol{\theta}_{n}^{t}-\boldsymbol{\theta}_{0}^{t}\right\|^{2}+2\left\|\nabla J(\boldsymbol{\theta}_{n}^{t})\right\|^{2}+2\left\|e_{t}\right\|^{2}
\end{align*}
\end{lemma}

\renewcommand\theequation{\ref{lemma-E-v}-\arabic{equation}}

\begin{proof}
We follow the suggestion of \citet{lei2017scsg2} to set $b_t = 1$ to deliver better theoretical results. However in our experiments, we do allow $b_t$ to be sampled from different values. With $b_t = 1$ and $\mu_{t}=\frac{1}{\left|\mathcal{G}_{t}\right|} \sum_{k \in \mathcal{G}_{t}} \mu_{t}^{(k)}$, we have the flowing definition according to Algorithm \ref{alg:FedPG-BR}:
\begin{align*}
v_{n}^{t} \triangleq g(\tau_{n}^{t} \mid \boldsymbol{\theta}_{n}^{t})-\omega(\tau_{n}^{t} \mid \boldsymbol{\theta}_{n}^{t}, \boldsymbol{\theta}_{0}^{t}) g(\tau_{n}^{t} \mid \boldsymbol{\theta}_{0}^{t})+u_{t} \stepcounter{equation}\tag{\theequation}
\end{align*}
which is the SCSG update step. Define $e_{t} \triangleq u_{t}-\nabla J\left(\boldsymbol{\theta}_{0}^{t}\right)$, we then have
\begin{align*}
\mathbb{E}_{\tau_n^t} [v_{n}^{t}] &=\nabla J(\boldsymbol{\theta}_{n}^{t})-\nabla J(\boldsymbol{\theta}_{0}^{t})+e_{t}+\nabla J(\boldsymbol{\theta}_{0}^{t}) \\
&=\nabla J(\boldsymbol{\theta}_{n}^{t})+e_{t} \stepcounter{equation}\tag{\theequation}\label{defn-Evt}
\end{align*}
Note that $\nabla J(\boldsymbol{\theta}_{n}^{t})-\nabla J(\boldsymbol{\theta}_{0}^{t}) = \mathbb{E}_{\tau_n^t}[g(\tau_{n}^{t} \mid \boldsymbol{\theta}_{n}^{t})-\omega(\tau_{n}^{t} \mid \boldsymbol{\theta}_{n}^{t}, \boldsymbol{\theta}_{0}^{t}) g(\tau_{n}^{t} \mid \boldsymbol{\theta}_{0}^{t})]
$ as we have showed that the importance weighting term results in unbiased estimation of the true gradient in Lemma \ref{lemma-unbiased-IS}. Then from $ \mathbb{E}\|X\|^2 = \mathbb{E}\|X-\mathbb{E}X\|^2 + \|\mathbb{E}X\|^2$, 
\begin{align*}
&\mathbb{E}_{\tau_n^t} [\|v_{n}^{t}\|^2] =\mathbb{E}_{\tau_n^t}\left\|v_{n}^{t}-\mathbb{E}_{\tau_n^t} [v_{n}^{t}]\right\|^{2}+\left\|\mathbb{E}_{\tau_n^t} [v_{n}^{t}]\right\|^{2}\\
&=\mathbb{E}_{\tau_n^t}\left\|g(\tau_{n}^{t} \mid \boldsymbol{\theta}_{n}^{t})-\omega(\tau_{n}^{t} \mid \boldsymbol{\theta}_{n}^{t}, \boldsymbol{\theta}_{0}^{t}) g(\tau_{n}^{t} \mid \boldsymbol{\theta}_{0}^{t})+u_{t}-(\nabla J(\boldsymbol{\theta}_{n}^{t})+e_{t})\right\|^{2}+\left\|\mathbb{E}_{\tau_n^t} [v_n^{t}\right]\|^{2}\\
&=\mathbb{E}_{\tau_n^t}\left\|g(\tau_{n}^{t} \mid \boldsymbol{\theta}_{n}^{t})-\omega(\tau_{n}^{t} \mid \boldsymbol{\theta}_{n}^{t}, \boldsymbol{\theta}_{0}^{t}) g(\tau_{n}^{t} \mid \boldsymbol{\theta}_{0}^{t})-(\nabla J(\boldsymbol{\theta}_{n}^{t})-\nabla J(\boldsymbol{\theta}_{0}^{t}))\right\|^{2}+\left\|\nabla J(\boldsymbol{\theta}_{n}^{t})+e_{t}\right\|^{2}\\
&\leq \mathbb{E}_{\tau_n^t}\left\|g(\tau_{n}^{t} \mid \boldsymbol{\theta}_{n}^{t})-\omega(\tau_{n}^{t} \mid \boldsymbol{\theta}_{n}^{t}, \boldsymbol{\theta}_{0}^{t}) g(\tau_{n}^{t} \mid \boldsymbol{\theta}_{0}^{t})\right\|^{2}+2\left\|\nabla J(\boldsymbol{\theta}_{n}^{t})\right\|^{2}+2\left\|e_{t}\right\|^{2} \stepcounter{equation}\tag{\theequation}\label{v-bound-1}\\ 
\end{align*}
where \eqref{v-bound-1} follows from $\mathbb{E}\|X-\mathbb{E}X\|^2 \leq \mathbb{E}\|X\|^2$ and Lemma \ref{lemma-usefulLemma}. Note that
\begin{align*}
& \quad \mathbb{E}_{\tau_n^t}\left\|g(\tau_{n}^{t} \mid \boldsymbol{\theta}_{n}^{t})   -\omega(\tau_{n}^{t} \mid \boldsymbol{\theta}_{n}^{t}, \boldsymbol{\theta}_{0}^{t}) g(\tau_{n}^{t} \mid \boldsymbol{\theta}_{0}^{t})\right\|^{2} \\
&=\mathbb{E}_{\tau_n^t} \|g(\tau_n^t \mid \boldsymbol{\theta}_{n}^t) + g(\tau_{n}^{t} \mid \boldsymbol{\theta}_{0}^{t}) - g(\tau_{n}^{t} \mid \boldsymbol{\theta}_{0}^{t}) -\omega(\tau_{n}^{t} \mid \boldsymbol{\theta}_{n}^{t}, \boldsymbol{\theta}_{0}^{t}) g(\tau_n^t \mid \boldsymbol{\theta}_{0}^t) \|^{2} \\
&=\mathbb{E}_{\tau_n^t} \|g(\tau_n^t \mid \boldsymbol{\theta}_{n}^t)  - g(\tau_{n}^{t} \mid \boldsymbol{\theta}_{0}^{t}) + (1-\omega(\tau_{n}^{t} \mid \boldsymbol{\theta}_{n}^{t}, \boldsymbol{\theta}_{0}^{t})) g(\tau_n^t \mid \boldsymbol{\theta}_{0}^t) \|^{2} 
\\
&\leq2 \mathbb{E}_{\tau_n^t}\left\|g(\tau_n^t \mid \boldsymbol{\theta}_{n}^t)-g(\tau_n^t \mid \boldsymbol{\theta}_{0}^t)\right\|^{2} + 2 \mathbb{E}_{\tau_n^t}\left\|(1-\omega(\tau_{n}^{t} \mid \boldsymbol{\theta}_{n}^{t}, \boldsymbol{\theta}_{0}^{t})) g(\tau_n^t \mid \boldsymbol{\theta}_{0}^t)\right\|^{2} \stepcounter{equation}\tag{\theequation}\label{v-bound-2} 
\end{align*}
where \eqref{v-bound-2} follows from Lemma \ref{lemma-usefulLemma}. Combining \eqref{v-bound-1} and \eqref{v-bound-2}, we have 
\begin{align*}
&\mathbb{E}_{\tau_n^t} [\|v_{n}^{t}\|^2] \leq 2 \mathbb{E}_{\tau_n^t}\left\|g(\tau_n^t \mid \boldsymbol{\theta}_{n}^t)-g(\tau_n^t \mid \boldsymbol{\theta}_{0}^t)\right\|^{2} \\
& \qquad \qquad \quad \quad +2 \mathbb{E}_{\tau_n^t}\left\|(1-\omega(\tau_n^t|\boldsymbol{\theta}_n^t,\boldsymbol{\theta}_0^t)) g(\tau_n^t \mid \boldsymbol{\theta}_{0}^t)\right\|^{2}  +2\left\|\nabla J(\boldsymbol{\theta}_{n}^{t})\right\|^{2}+2\left\|e_{t}\right\|^{2} \\
                                     &\leq 2L_g\left\|\boldsymbol{\theta}_{n}^{t}-\boldsymbol{\theta}_{0}^{t}\right\|^{2}+2 C_{g}^{2} \mathbb{E}_{\tau_n^t}\|(1-\omega(\tau_n^t|\boldsymbol{\theta}_n^t,\boldsymbol{\theta}_0^t))\|^{2}+2\left\|\nabla J(\boldsymbol{\theta}_{n}^{t})\right\|^{2}+2\left\|e_{t}\right\|^{2}  \stepcounter{equation}\tag{\theequation}\label{v-bound-3}   \\
                                     &\leq 2L_g\left\|\boldsymbol{\theta}_{n}^{t}-\boldsymbol{\theta}_{0}^{t}\right\|^{2}+2C_g^{2} C_{w}\left\|\boldsymbol{\theta}_{n}^{t}-\boldsymbol{\theta}_{0}^{t}\right\|^{2}+2\left\|\nabla J(\boldsymbol{\theta}_{n}^{t})\right\|^{2}+2 \| e_{t}\|^{2}  \stepcounter{equation}\tag{\theequation}\label{v-bound-4} \\
                                     &=(2 L_{g}+2C_g^{2} C_{w})\left\|\boldsymbol{\theta}_{n}^{t}-\boldsymbol{\theta}_{0}^{t}\right\|^{2}+2\left\|\nabla J(\boldsymbol{\theta}_{n}^{t})\right\|^{2}+2\left\|e_{t}\right\|^{2} \stepcounter{equation}\tag{\theequation}\label{v-bound-final}
\end{align*}
where \eqref{v-bound-3} is from Lemma \ref{proposition-uai-grad} and \eqref{v-bound-4} follows from Lemma \ref{lemma-uai-variance-IS}
\end{proof}

\begin{lemma}
    \label{lemma-eta-E}
    \begin{align*}
        \eta_{t} \mathbb{E}\left\langle e_{t}, \mathbb{E} \nabla J(\tilde{\boldsymbol{\theta}}_{t})\right\rangle=\frac{1}{B_{t}} \mathbb{E}\left\langle e_{t}, \tilde{\boldsymbol{\theta}}_{t}-\tilde{\boldsymbol{\theta}}_{t-1}\right\rangle-\eta_{t} \mathbb{E}\left\|e_{t}\right\|^{2}
    \end{align*}    
\end{lemma}

\begin{proof}
Consider $M_n^t = \langle e_t, \boldsymbol{\theta}_n^t - \boldsymbol{\theta}_0^t \rangle$. We have
\begin{align*}
    M_{n+1}^t - M_n^t = \langle e_t, \boldsymbol{\theta}_{n+1}^t - \boldsymbol{\theta}_n^t \rangle = \eta_t \langle e_t, v_n^t \rangle
\end{align*}

Taking expectation with respect to $\tau_n^t$, we have
\begin{align*}
 \mathbb{E}_{\tau_n^t}\left[M_{n+1}^{t}-M_{n}^{t}\right] &=\eta_{t}\left\langle e_{t},  \mathbb{E}_{\tau_n^t}[v_{n}^{t}]\right\rangle \\
&=\eta_{t}\left\langle e_{t}, \nabla J(\boldsymbol{\theta}_{n}^{t})\right\rangle+\eta_{t}\left\|e_{t}\right\|^{2}
\end{align*}
following from \eqref{defn-Evt}. Use $\mathbb{E}_t$ to denote the expectation with respect to all trajectories $\{\tau^t_1, \tau^t_2, ...\}$, given $N_t$. Since $\{\tau^t_1, \tau^t_2, ...\}$ are independent of $N_t$, $\mathbb{E}_t$ is equivalently the expectation with respect to $\{\tau^t_1, \tau^t_2, ...\}$. We have 
\begin{align*}
    \mathbb{E}_{t}[M_{n+1}^{t}-M_{n}^{t}] = \eta_{t}\left\langle e_{t}, \mathbb{E}_{t}\nabla J(\boldsymbol{\theta}_{n}^{t})\right\rangle+\eta_{t}\left\|e_{t}\right\|^{2}
\end{align*}
Taking $n = N_t$ and denoting $\mathbb{E}_{N_t}$ the expectation w.r.t. $N_t$, we have
\begin{align*}
    \mathbb{E}_{N_t}\mathbb{E}_t(M_{N_{t}+1}^{t}-M_{N_{t}}^{t})=\eta_{t} \langle e_{t}, \mathbb{E}_{N_t}\mathbb{E}_t\nabla J(\boldsymbol{\theta}_{N_{t}}^{t}) \rangle +\eta_{t}\left\|e_{t}\right\|^{2}.
\end{align*}

Using Fubini’s theorem, Lemma \ref{lemma-GeoSampling} and using the fact $\boldsymbol{\theta}_{N_t}^t = \tilde{\boldsymbol{\theta}}_t$ and $\boldsymbol{\theta}_{0}^t = \tilde{\boldsymbol{\theta}}_{t-1}$, 
\begin{align*}
    \mathbb{E}_{N_t}\mathbb{E}_t(M_{N_{t}+1}^{t}-M_{N_{t}}^{t}) &= -\mathbb{E}_{t}\mathbb{E}_{N_t}(M_{N_{t}}^{t}-M_{N_{t}+1}^{t}) \\
                                                                &= -(\frac{1}{B_t/(B_t+1)}-1)(M_0^t - \mathbb{E}_{N_t}\mathbb{E}_t M_{N_t}^t) \\
                                                                &= \frac{1}{B_{t}} \mathbb{E}_{N_t}\mathbb{E}_t\left\langle e_{t}, \tilde{\boldsymbol{\theta}}_{t}-\tilde{\boldsymbol{\theta}}_{t-1}\right\rangle \\
                                                                &=\eta_{t}\left\langle e_{t}, \mathbb{E}_{N_t}\mathbb{E}_t \nabla J(\boldsymbol{\theta}_{N_{t}}^{t})\right\rangle+\eta_{t}\left\|e_{t}\right\|^{2}
\end{align*}
Taking expectation with respect to the whole past yields the lemma.
\end{proof}

\begin{lemma}
    \label{lemma-neg-2eta-E}
    \begin{align*}
        -2\eta_t\mathbb{E} \langle e_t, \tilde{\boldsymbol{\theta}}_t - \tilde{\boldsymbol{\theta}}_{t-1} \rangle \leq \left[- \frac{1}{B_t} + \eta_t^2(2L_g + 2C_g^2C_w)\right]\mathbb{E}\|\tilde{\boldsymbol{\theta}}_t - \tilde{\boldsymbol{\theta}}_{t-1}\|^2 + 2\eta_t^2\mathbb{E}\|e_t\|^2 \\ + 2\eta_t \mathbb{E} \langle \nabla J(\tilde{\boldsymbol{\theta}}_t, \tilde{\boldsymbol{\theta}}_t - \tilde{\boldsymbol{\theta}}_{t-1} \rangle  + 2\eta^2_t\mathbb{E}\|\nabla J(\tilde{\boldsymbol{\theta}}_t)\|^2 
    \end{align*}    
\end{lemma}

\renewcommand\theequation{\ref{lemma-neg-2eta-E}-\arabic{equation}}
\begin{proof}
We have from the update equation $\boldsymbol{\theta}_{n+1}^t = \boldsymbol{\theta}_n^t + \eta_t v_n^t$, then,
\begin{align*}
    \mathbb{E}_{\tau_n^t}\|\boldsymbol{\theta}_{n+1}^t - \boldsymbol{\theta}_0^t\|^2 &= \mathbb{E}_{\tau_n^t}\|\boldsymbol{\theta}_{n}^t + \eta_t v_n^t - \boldsymbol{\theta}_0^t\|^2 \\
                                                            &= \|\boldsymbol{\theta}_n^t - \boldsymbol{\theta}_0^t\|^2 + \eta_t^2\mathbb{E}_{\tau_n^t}\|v_n^t\|^2 + 2\eta_t \langle\mathbb{E}_{\tau_n^t}[v_n^t], \boldsymbol{\theta}_n^t - \boldsymbol{\theta}_0^t\rangle \\
                                                            &\leq \|\boldsymbol{\theta}_n^t - \boldsymbol{\theta}_0^t\|^2 + \eta_t^2[(2L_g + 2C_g^2C_w)\|\boldsymbol{\theta}_n^t-\boldsymbol{\theta}_0^t\|^2 + 2\|\nabla J(\boldsymbol{\theta}_n^t)\|^2 +2\|e_t\|^2] \\
                                                            &  \qquad \qquad +2\eta_t\langle e_t, \boldsymbol{\theta}_n^t-\boldsymbol{\theta}_0^t \rangle +2\eta_t\langle\nabla J(\boldsymbol{\theta}_n^t), \boldsymbol{\theta}_n^t-\boldsymbol{\theta}_0^t \rangle \stepcounter{equation}\tag{\theequation}\label{lemma-neg-2eta-E-1}\\
                                                            &= [1+\eta_t^2(2L_g + 2C_g^2C_w)]\|\boldsymbol{\theta}_n^t-\boldsymbol{\theta}_0^t\|^2 + 2\eta_t\langle\nabla J(\boldsymbol{\theta}_n^t), \boldsymbol{\theta}_n^t-\boldsymbol{\theta}_0^t \rangle \\
                                                            &  \qquad \qquad +2\eta_t\langle e_t, \boldsymbol{\theta}_n^t-\boldsymbol{\theta}_0^t \rangle  + 2\eta_t^2\|\nabla J(\boldsymbol{\theta}_n^t)\|^2 + 2\eta_t^2\|e_t\|^2 
\end{align*}
where \eqref{lemma-neg-2eta-E-1} follows the result of \eqref{v-bound-final}. Use $\mathbb{E}_t$ to denote the expectation with respect to all trajectories $\{\tau^t_1, \tau^t_2, ...\}$, given $N_t$. Since $\{\tau^t_1, \tau^t_2, ...\}$ are independent of $N_t$, $\mathbb{E}_t$ is equivalently the expectation with respect to $\{\tau^t_1, \tau^t_2, ...\}$. We have 
\begin{align*}
    \mathbb{E}_{t}\|\boldsymbol{\theta}_{n+1}^t - \boldsymbol{\theta}_0^t\|^2 \leq [1+\eta_t^2(2L_g + 2C_g^2C_w)]\mathbb{E}_{t}\|\boldsymbol{\theta}_n^t-\boldsymbol{\theta}_0^t\|^2 + 2\eta_t\mathbb{E}_{t}\langle\nabla J(\boldsymbol{\theta}_n^t), \boldsymbol{\theta}_n^t-\boldsymbol{\theta}_0^t \rangle \\  \qquad \qquad +2\eta_t\mathbb{E}_{t}\langle e_t, \boldsymbol{\theta}_n^t-\boldsymbol{\theta}_0^t \rangle + 2\eta_t^2\mathbb{E}_{t}\|\nabla J(\boldsymbol{\theta}_n^t)\|^2 + 2\eta_t^2\|e_t\|^2
\end{align*}
Now taking $n = N_t$ and denoting $E_{N_t}$ the expectation w.r.t. $N_t$ we have
\begin{align*}
    &-2 \eta_{t} \mathbb{E}_{N_{t}} \mathbb{E}_{t} \left \langle  e_{t}, \boldsymbol{\theta}_{N_t}^t-\boldsymbol{\theta}_{0}^t\right\rangle \\
    &\leq [1+\eta_{t}^{2}(2L_g + 2C_g^2C_w)] \mathbb{E}_{N_{t}} \mathbb{E}_{t}\left\|\boldsymbol{\theta}_{N_{t}}^{t}-\boldsymbol{\theta}_{0}^{t}\right\|^{2}-\mathbb{E}_{N_{t}} \mathbb{E}_{t}\left\|\boldsymbol{\theta}_{N_{t}+1}^{t}-\boldsymbol{\theta}_{0}^{t}\right\|^{2}  \\
                             & \qquad \qquad +2 \eta_{t} \mathbb{E}_{N_{t}} \mathbb{E}_{t}\left\langle\nabla J(\boldsymbol{\theta}_{N_t}^t), \boldsymbol{\theta}_{N_t}^t-\boldsymbol{\theta}_{0}^t\right\rangle+2 \eta_{t}^{2} \mathbb{E}_{N_{t}} \mathbb{E}_{t}\left\|\nabla J(\boldsymbol{\theta}_{N_t}^t)\right\|^{2}+2 \eta_{t}^{2}\left\|e_{t}\right\|^{2} \\
                                              &  =\left[-\frac{1}{B_{t}}+\eta_{t}^{2} (2L_g + 2C_g^2C_w)\right] \mathbb{E}_{N_{t}} \mathbb{E}_{t}\left\|\boldsymbol{\theta}_{N_t}^t-\boldsymbol{\theta}_{0}^t\right\|^{2} \\
                            &\qquad \qquad  +2 \eta_{t} \mathbb{E}_{N_{t}} \mathbb{E}_{t}\left\langle\nabla J(\boldsymbol{\theta}_{N_t}^t), \boldsymbol{\theta}_{N_t}^t-\boldsymbol{\theta}_{0}^t\right\rangle+2 \eta_{t}^{2} \mathbb{E}_{N_{t}} \mathbb{E}_{t}\left\|\nabla J(\boldsymbol{\theta}_{N_t}^t)\right\|^{2}+2 \eta_{t}^{2}\left\|e_{t}\right\|^{2} \stepcounter{equation}\tag{\theequation}\label{lemma-neg-2eta-E-2}
\end{align*}
where \eqref{lemma-neg-2eta-E-2} follows Lemma \ref{lemma-GeoSampling} using Fubini's theorem. Rearranging, replacing $\boldsymbol{\theta}_{N_t}^t = \tilde{\boldsymbol{\theta}}_t$ and $\boldsymbol{\theta}_0^t = \tilde{\boldsymbol{\theta}}_{t-1}$ and taking expectation w.r.t the whole past yields the lemma.
\end{proof}


\begin{lemma}[Pinelis' inequality \citep{pinelis1994optimumMartingale}; Lemma 2.4 \citep{alistarh2018byzantine}]\label{lemma-martingale} Let the sequence of random variables $X_1, X_2,...,X_N \in \mathbb{R}^d$ represent a random process such that we have $\mathbb{E}[X_n|X_1,...,X_{n-1}]$ and $\|X_n\| \leq M$. Then,
\begin{align*}
    \mathbb{P}\left[\left\|X_{1}+\ldots+X_{N}\right\|^{2} \leq 2 \log (2 / \delta) M^{2} N\right] \geq 1-\delta
\end{align*}
\end{lemma}

\begin{lemma}[Adapted from \citep{khanduri2019byzantine}]
    \label{lemma-error-bound}
    If we choose $\delta$ and $B_t$ in Algorithm \ref{alg:FedPG-BR} such that: \\
    (i) $\mathrm{e}^{\frac{\delta B_{t}}{2(1-2 \delta)}} \leq \frac{2 K}{\delta} \leq \mathrm{e}^{\frac{B_{t}}{2}}$ \\
    (ii) $\delta \leq \frac{1}{5 K B_{t}}$ \\
    then we have the following bound for $\mathbb{E}\|e_t\|^2$:
    
    \begin{align*}
        \mathbb{E}\left\|e_{t}\right\|^{2} \leq \frac{4 \sigma^{2}}{(1-\alpha)^{2} K B_{t}}+\frac{48 \alpha^{2} \sigma^{2} V}{(1-\alpha)^{2} B_{t}}
    \end{align*}    
\end{lemma}
\begin{proof}
The proof of this lemma is similar to that of Lemma 7 of \citet{khanduri2019byzantine}. The key difference lays on the base conditions used to define the probabilistic events.

In FedPG-BR, the following refined conditions (results of Claims \ref{claim:R2-ensures-all-good-agents-are-included} and \ref{claim2:distance-in-filtering}) are used, 
\begin{align*}
    \|\mu_t^{\operatorname{mom}} - \nabla J(\boldsymbol{\theta})\| \leq \sigma, \|\mu_t^{(k)} - \mu_t^{\operatorname{mom}}\| \leq 2\sigma, \|\mu_t^{(k)} - \nabla J(\boldsymbol{\theta})\| \leq 3\sigma, \forall k \in \mathcal{G}_t
\end{align*}
whereas \citet{khanduri2019byzantine} needs the following:
\begin{align*}
    \|\mu_t^{\operatorname{med}} - \nabla J(\boldsymbol{\theta})\| \leq 3\sigma, \|\mu_t^{(k)} - \mu_t^{\operatorname{med}}\| \leq 4\sigma, \|\mu_t^{(k)} - \nabla J(\boldsymbol{\theta})\| \leq 7\sigma, \forall k \in \mathcal{G}_t
\end{align*}
The detailed proof of Lemma \ref{lemma-error-bound} can be obtained following the derivation of Lemma 7 of \citet{khanduri2019byzantine} by modifying the base conditions.
\end{proof}
\begin{lemma}
\label{lemma-GeoSampling}
If $N \sim Geom(\Gamma)$ for $\Gamma > 0$. Then for any sequence $D_0, D_1, . . . $ with $ \mathbb{E}\|D_N \| \leq~\infty$, we
have
\begin{align*}
    \mathbb{E}\left[D_{N}-D_{N+1}\right]=(\frac{1}{\Gamma}-1)(D_{0}-\mathbb{E} D_{N})
\end{align*}
\end{lemma}
\begin{proof}
The proof can be found in \citet{lei2017scsg2}.
\end{proof}

\begin{lemma}[Young's inequality (Peter-Paul inequality)]\label{lemma:Youngs-inequality}
For all real numbers $a$ and $b$ and all $\beta > 0$, we have
\begin{align*}
    ab \leq \frac{a^2}{2\beta} + \frac{\beta b^2}{2}
\end{align*}
\end{lemma}

\renewcommand\theequation{\ref{main-theorem}-\arabic{equation}}
\section{Experimental details}\label{EXP-Settings-appendix}
\subsection{Hyperparameters}
We follow the setups of SVRPG \citep{papini2018stochastic} to parameterize the policies using neural networks. For all the algorithms under comparison in the experiments (Section \ref{section:experiments}), Adam\citep{kingma2014adam} is used as the gradient optimizer. The 10 random seeds are $[0\!-\!9]$. All other hyperparameters used in all the experiments are reported in Table \ref{tab:params}. 

\begin{table}[ht]
\centering
\small
\renewcommand\arraystretch{1.15}
\caption{Hyperparameters used in the experiments.}
\hspace{0.1mm}
\label{tab:params}
    \resizebox{\textwidth}{!}{%
    \begin{tabular}{@{}ccccc@{}}
    \hline
    Hyperparameters & Algorithms & CartPole-v1 & LunarLander-v2 & HalfCheetach-v2 \\ \hline
    NN policy & - & Categorical MLP & Categorical MLP & Gaussian MLP \\
    NN hidden weights & - & 16,16 & 64,64 & 64,64 \\
    NN activation & - & ReLU & Tanh & Tanh \\
    NN output activation & - & Tanh & Tanh & Tanh \\
    Step size (Adam) $\eta$ & - & 1e-3 & 1e-3 & 8e-5 \\
    Discount factor $\gamma$ & -  & 0.999 & 0.990 & 0.995 \\
    Maximum trajectories & -  & 5000 & 10000 & 10000 \\
    Task horizon $H$ (for training) & - & 500 & 1000 & 500 \\
    Task horizon $H$ (for test) & - & 500 & 1000 & 1000 \\
    $\alpha$ (for practical setup) & - & 0.3 & 0.3 & 0.3 \\
    Number of runs & - & 10 & 10 & 10 \\ \cline{2-5} 
    \multirow{3}{*}{Batch size $B_t$} & GPOMDP & 16 & 32 & 48 \\
     & SVRPG & 16 & 32 & 48 \\
     & FedPG-BR & sampled from $[12,20]$ & sampled from $[26,38]$ & sampled from $[46,50]$ \\ \cline{2-5} 
    \multirow{3}{*}{Mini-Batch size $b_t$} & GPOMDP & - & - & - \\
     & SVRPG & 4 & 8 & 16 \\
     & FedPG-BR & 4 & 8 & 16 \\ \cline{2-5} 
    \multirow{3}{*}{Number of steps $N_t$} & GPOMDP & 1 & 1 & 1 \\
     & SVRPG & 3 & 3 & 3 \\
     & FedPG-BR & $N_t \sim Geom(\frac{B_t}{B_t + b_t})$ &  $N_t \sim Geom(\frac{B_t}{B_t + b_t})$ &  $N_t \sim Geom(\frac{B_t}{B_t + b_t})$\\ \cline{2-5} 
    \multirow{3}{*}{\begin{tabular}[c]{@{}c@{}}Variance bound $\sigma$ \\ (Estimated by server)\end{tabular}
    } & GPOMDP & - & - & - \\
     & SVRPG & - & - & - \\
     & FedPG-BR & 0.06 & 0.07 & 0.9 \\ \cline{2-5} 
    \multirow{3}{*}{Confidence parameter $\delta$} & GPOMDP & - & - & - \\
     & SVRPG & - & - & - \\
     & FedPG-BR & 0.6 & 0.6 & 0.6 \\ \hline
    \end{tabular}
    }
\end{table}
%
%
%
    %
    
\subsection{Computing Infrastructure}\label{appendix:computing-infrastruture}
All experiments are conducted on a computing server {without GPUs}. The server is equipped with 14 cores (28 threads) \textit{Intel(R) Core(TM) i9-10940X CPU @ 3.30GHz} and 64G memory. The average runtime for each run of FedPG-BR (K=10 B=3) is 2.5 hours for the CartPole task, 4 hours for the HalfCheetah task, and 12 hours for the LunarLander task.    
\section{Additional experiments}\label{additional-experiments}
\subsection{Performance of FedPG-BR in practical systems with $\alpha > 0$ for the CartPole and the LunarLander tasks}\label{appendix:exp-exp2=lunarlander}
The results for the CartPole and the LunarLander tasks which yield the same insights as discussed in experiments (Section \ref{section:experiments}) are plotted in Figure \ref{fig:exp2-cartpole} and Figure \ref{fig:exp2-lunarlander}. As discussed earlier, for both GPOMDP and SVRPG, the federation of more agents in practical systems which are subject to the presence of Byzantine agents, i.e., random failures or adversarial attacks, causes the performance of their federation to be worse than that in the single-agent setting. In particular, RA agents (middle figure) and SF agents (right figure) render GPOMDP and SVRPG unlearnable, i.e., unable to converge at all. This is in contrast to the performance of FedPG-BR. That is, FedPG-BR ($K=10 B=3$) is able to deliver superior performances even in the presence of Byzantine agents for all three tasks: CartPole (Figure \ref{fig:exp2-cartpole}), LunarLander (Figure \ref{fig:exp2-lunarlander}), and HalfCheetah (Figure \ref{fig:exp2} in Section \ref{section:experiments}). This provides an assurance on the reliability of our FedPG-BR algorithm to promote its practical deployment, and significantly improves the practicality of FRL.
\begin{figure}[t]
    \centering
    \includegraphics[height=1.3in] {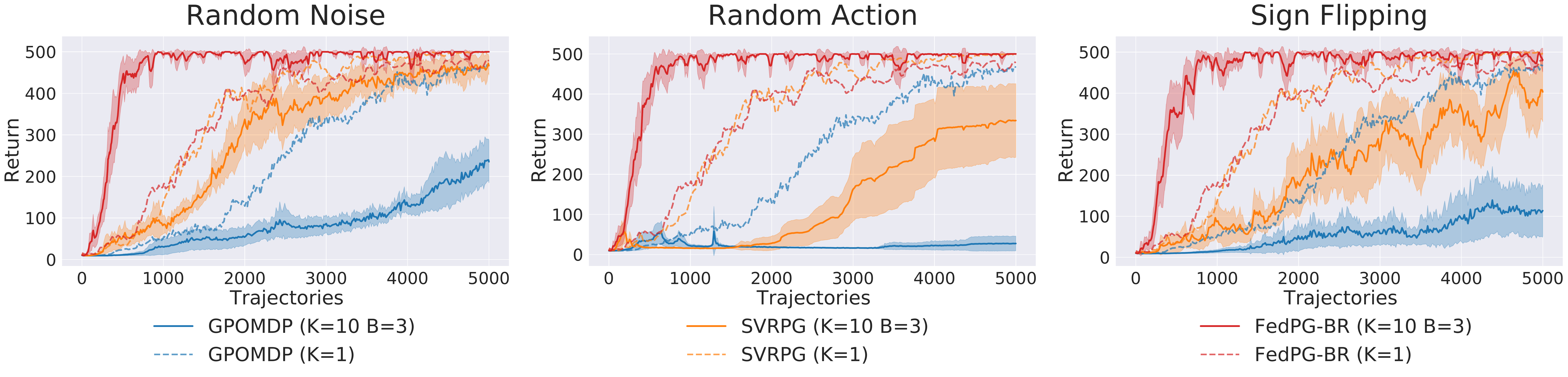}
    \caption{Performance of FedPG-BR in practical systems with $\alpha>0$ for CartPole. 
    Each subplot corresponds to a different type of Byzantine failure exercised by the 3 Byzantine agents.}
    \label{fig:exp2-cartpole}
\end{figure}
\begin{figure}[!t]
    \centering
    \includegraphics[height=1.25in] {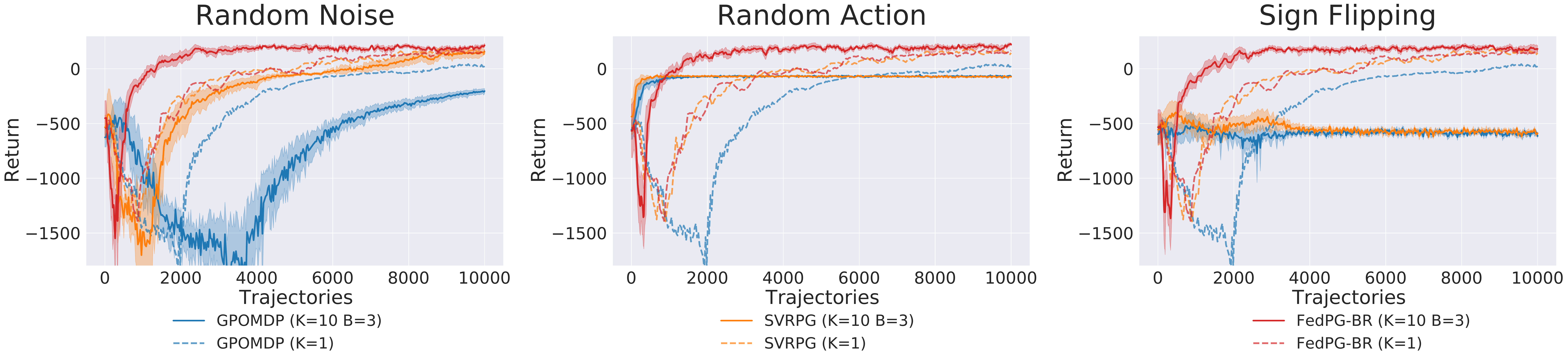}
    \caption{Performance of FedPG-BR in practical systems with $\alpha>0$ for LunarLander. 
    Each subplot corresponds to a different type of Byzantine failure exercised by the 3 Byzantine agents.}
    \label{fig:exp2-lunarlander}
\end{figure}

\subsection{Performance of FedPG-BR against the Variance Attack}
We have discussed in Section~\ref{subsection:technical-challenges} where the high variance in PG estimation renders the FRL system vunlnerable to variance-based attacks such as the Variance Attack (VA) proposed by \citet{baruch2019Byzantine-VA-attack}. The VA attackers collude together to estimate the population mean and the standard-deviation of gradients at each round, and move the mean by the largest value such that their values are still within the population variance. Intuitively, this non-omniscient attack works by exploiting the high variance in gradient estimation of the population and crafting values that contribute most to the population variance, hence gradually shifting the population mean. According to \citet{cao2019adversarialAttack-AV}, existing defenses will fail to remove those non-omniscient attackers and the convergence will be significantly worsened if the population variance is large enough. 
 
\begin{figure}[]
    \centering
    \includegraphics[height=1.2in] {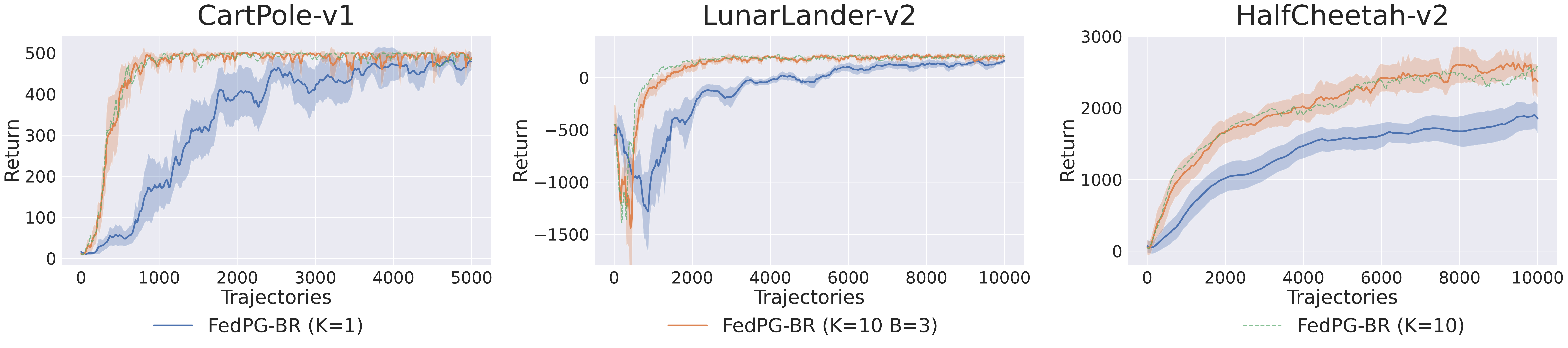}
    \caption{Performance of FedPG-BR in practical systems with $\alpha>0$ for CartPole. 
    Among the $K=10$ participating agents, 3 Byzantine agents are colluding together to launch the VA attack.}
    \label{fig:exp-VA}
\end{figure}

We are thus motivated to look for solutions that theoretically reduce the variance in policy gradient estimation. Inspired by the variance-reduced policy gradient works \citep[e.g.,][]{papini2018stochastic,xu2020improvedUAI}, we adapt the SCSG optimization \citep{lei2017scsg2} to our federated policy gradient framework for a refined control over the estimation variance. Through our adaptation, we are able to control the variance by the semi-stochastic gradient (line 11 in Algorithm~\ref{alg:FedPG-BR}), hence resulting in the fault-tolerant FRL system that can defend the VA attackers. Each plot in Figure~\ref{fig:exp-VA} shows the experiment for each of the three tasks correspondingly, where 3 Byzantine agents are implemented as the VA attackers \citep{cao2019adversarialAttack-AV} ($z^{max}$ is $0.18$ in our setup). We again include the corresponding single-agent performance ($K=1$) and the federation of 10 good agents ($K=10$) in the plots for reference. The results show that in all three tasks, FedPG-BR ($K=10 B=3$) still manages to significantly outperform FedPG-BR ($K=1$) in the single-agent setting. Furthermore, the performance of FedPG-BR ($K=10 B=3$) is barely worsened compared with FedPG-BR ($K=10$) with 10 good agents. This shows that, with the adaptation of SCSG, our fault-tolerant FRL system can perfectly defend the VA attack from the literature, which further corroborates our analysis on our Byzantine filtering step (Section~\ref{subsection:algorithm-description}) showing that if gradients from Byzantine agents are not filtered out, their impact is limited since their maximum distance to $\nabla J(\boldsymbol{\theta}^t_0)$ is bounded by $3\sigma$ (Claim~\ref{claim2:distance-in-filtering}).


\subsection{Environment Setup}
On a Linux system, navigate into the root directory of this project and execute the following commands: 
\begin{lstlisting}[language=bash]
    $ conda create -n FT-FRL pytorch=1.5.0
    $ conda activate FT-FRL
    $ pip install -r requirements.txt
    $ cd codes
\end{lstlisting}

To run experiments in HalfCheetah, a mujoco license\footnote{\url{http://www.mujoco.org}} is required. After obtaining the license, install the mujoco-py library by following the instructions from OpenAI.\footnote{\url{https://github.com/openai/mujoco-py}}

\subsection{Examples}
To reproduce the results of FedPG-BR ($K=10$) in Figure \ref{fig:exp1} for the HalfCheetah task, run the following command:
\begin{lstlisting}[language=bash]
    $ python run.py --env_name HalfCheetah-v2 --FT_FedPG
    --num_worker 10 --num_Byzantine 0 
    --log_dir ./logs_HalfCheetah --multiple_run 10 
    --run_name HalfCheetah_FT-FRL_W10B0
\end{lstlisting}
To reproduce the results of FedPG-BR ($K=10\ B=3$) in Figure \ref{fig:exp2} where 3 Byzantine agents are Random Noise in the HalfCheetah task environment, run the following command:
\begin{lstlisting}[language=bash]
    $ python run.py --env_name CartPole-v1 --FT_FedPG
    --num_worker 10 --num_Byzantine 3 
    --attack_type random-noise 
    --log_dir ./logs_Cartpole --multiple_run 10 
    --run_name Cartpole_FT-FRL_W10B3
\end{lstlisting}
Replace \lstinline{`--FT_FedPG'} with \lstinline{`--SVRPG'} for the results of SVRPG in the same experiment. All results including all statistics will be logged into the directory indicated by \lstinline{`--log_dir`}, which can be visualized in tensorboard.

To visualize the behavior of the learnt policy, run the experiment in evaluation mode with rendering option on. For example:
\begin{lstlisting}[language=bash]
    $ python run.py --env_name CartPole-v1 --FT_FedPG
    --eval_only --render 
    --load_path PATH_TO_THE_SAVED_POLICY_MODEL
\end{lstlisting}

\end{document}